\pdfoutput=1

\documentclass[11pt, letterpaper]{article}
\usepackage{natbib}
\usepackage[utf8]{inputenc} %
\usepackage[T1]{fontenc}    %
\usepackage[margin=1in]{geometry}
\usepackage{setspace}
\newlength\tindent
\renewcommand{\indent}{\hspace*{\tindent}}

\usepackage[parfill]{parskip}

\usepackage{microtype}
\usepackage{graphicx}
\usepackage{booktabs} %

\usepackage{amsmath,amsfonts,bm,dsfont,amsthm}

\def\eqref#1{equation~\ref{#1}}

\def\1{\bm{1}}

\DeclareMathAlphabet{\mathsfit}{\encodingdefault}{\sfdefault}{m}{sl}
\SetMathAlphabet{\mathsfit}{bold}{\encodingdefault}{\sfdefault}{bx}{n}

\newcommand{\E}{\mathbb{E}}

\DeclareMathOperator*{\argmin}{arg\,min}

\usepackage{tcolorbox}
\definecolor{grey}{rgb}{0.33, 0.33, 0.33}

\newcommand{\squishlist}{
\begin{list}{{{\small{$\bullet$}}}}
{\setlength{\itemsep}{1pt}      \setlength{\parsep}{0pt}
\setlength{\topsep}{-2pt}       \setlength{\partopsep}{0pt}
\setlength{\leftmargin}{1em} \setlength{\labelwidth}{1em}
\setlength{\labelsep}{0.5em} } }
\newcommand{\squishend}{  \end{list}  }

\makeatletter
\renewcommand*\env@matrix[1][*\c@MaxMatrixCols c]{%
  \hskip -\arraycolsep
  \let\@ifnextchar\new@ifnextchar
  \array{#1}}
\makeatother

\usepackage{subcaption}

\usepackage{hyperref}
\usepackage{url}
\usepackage{hyperref}       %
\usepackage{url}            %
\usepackage{booktabs}       %
\usepackage{amsfonts}       %
\usepackage{nicefrac}       %
\usepackage{lipsum}
\usepackage{fancyhdr} %
\usepackage{subcaption}

\usepackage{hyperref}

\usepackage{amsmath}
\usepackage{amssymb}
\usepackage{mathtools}
\usepackage{amsthm}

\newcommand\hcancel[2][black]{\setbox0=\hbox{$#2$}%
\rlap{\raisebox{.45\ht0}{\textcolor{#1}{\rule{\wd0}{1pt}}}}#2} \usepackage{graphicx} 
\usepackage{amssymb}
\usepackage{wrapfig}

\usepackage{amsmath}
\newcommand{\D}{\mathcal D}
\newcommand{\K}{\mathcal K}

\newcommand{\infl}{\text{infl}}
\newcommand{\p}{\mathbb P}

\newcommand{\err}{\mathsf{err}}

\newtheorem{theorem}{Theorem}

\newtheorem{proposition}{Proposition}

\newtheorem{definition}{Definition}

\newcommand{\para}[1]{{\vspace{2pt} \noindent \textbf{#1}}}

\newcommand{\eg}{{e.g.\ }}
\newcommand{\ie}{{i.e.\ }}
\newcommand{\etc}{{etc.}}

\usepackage[capitalize,noabbrev]{cleveref}

\usepackage[textsize=tiny]{todonotes}

\title{On the Cause of Unfairness: A Training Sample Perspective}

\author{Yuanshun Yao \quad Yang Liu\\
~\\
\{kevin.yao, yang.liu01\}@bytedance.com\\
ByteDance Research, San Jose, CA}
\date{}

\begin{document}

\maketitle
\begin{abstract}
\noindent Identifying the causes of a model's unfairness is an important yet relatively unexplored task. We look into this problem through the lens of training data -- the major source of unfairness. We ask the following questions: How would the unfairness of a model change if its training samples (1) were collected from a different (e.g. demographic) group, (2) were labeled differently, or (3) whose features were modified? In other words, we quantify the influence of training samples on unfairness by counterfactually changing samples based on predefined concepts, i.e. data attributes such as features ($X$), labels ($Y$), and sensitive attributes ($A$). Our framework not only can help practitioners understand the observed unfairness and mitigate it by repairing their training data, but also leads to many other applications, e.g. detecting mislabeling, fixing imbalanced representations, and detecting fairness-targeted poisoning attacks.

\end{abstract}

\section{Introduction}

A fundamental question in machine learning fairness is: What causes the unfairness? Without the answer, it is hard to understand and fix the problem. In practice, this is also one of the first questions the practitioners would ask after measuring fairness. Although the question sounds simple, it is hard to identify the exact source of unfairness in the machine learning pipeline.\footnote{As admitted by many leading fairness practitioners, \eg Meta blog~\citep{meta} states: ``Unfairness in an AI model could have many possible causes, including not enough training data, a lack of features, a misspecified target of prediction, or a measurement error in the input features. Even for the most sophisticated AI researchers and engineers, these problems are not straightforward to fix.''}

The problem is difficult since the sources of unfairness are many, including data sampling bias or under-representation \citep{chai2022fairness,zhu2021rich,celis2021fair,bagdasaryan2019differential}, data labeling bias \citep{wang2021fair,wu2022fair,fogliato2020fairness}, model architecture (or feature representation) \citep{adel2019one,madras2018learning,zemel2013learning,song2019learning,xing2021fairness,li2021dyadic,song2021deep,li2020deep}, distribution shift \citep{ding2021retiring,chen2022fairness,rezaei2021robust,giguere2022fairness} \etc{}

In this work, we tackle this problem by studying the most important and obvious source of bias: training samples. The rationale is simple: if the model's training samples are biased, it would be unlikely that the model can still remain unbiased. 

To illustrate how we choose which perspective of the training data to study, consider the conventional fairness measure Equality of Opportunity~\citep{hardt2016equality}:
\begin{equation}
\big| \mathbb{P}(h_\theta(X) = 1 | A = 0, Y=1)-\mathbb{P}(h_\theta(X) = 1 | A = 1, Y=1) \big|
\end{equation}
where $X$ is feature, $Y$ is label, $A$ is sensitive attribute, and $h_\theta(X)$ is model $\theta$'s predicted label given $X$. It is clear the underlying fairness is related to feature $X$, label $Y$, and sensitive attribute $A$. We, therefore, ask the following logical question regarding how training samples would impact the model's unfairness: How a model's (un)fairness would change if its training samples (1) were collected from a different (\eg demographic) group (i.e. sensitive attribute $A$ is changed), (2) were labeled differently, (i.e. label $Y$ is changed) or (3) were modified for some features (i.e. feature $X$ is changed)?

However, the relationship between a training sample's attribute (i.e. $X$, $Y$, or $A$) and the measured fairness cannot be derived precisely. For example, if we change a training sample's feature $X$, the trained model weights $\theta$ will also change, and how it would impact the measured fairness on $\theta$ is too complex to derive precisely.

To this end, we propose a framework based on \textit{influence function}~\citep{cook1982residuals,koh2017understanding} to measure a training attribute's impact on fairness. Influence function provides a tool to quantify the change in the underlying fairness with respect to a training sample. We adapt it to define the influence on fairness measure w.r.t a training \textit{concept}, \ie a categorical variable that describes data property like data attribute.

For example, we can choose the concept to be the sensitive group attribute $A$ and counterfactually\footnote{We use the word ``counterfactual'' in its literal sense, \ie being different from the factual world, in the same \textit{empirical} fashion of the counterfactual example or the counterfactual explanation~\citep{verma2020counterfactual,ustun2019actionable,roese1997counterfactual} rather than in the rigorous and theoretical sense of ``counterfactual'' in the causal inference. Our work does \textit{not} belong to the area of causal inference.} change it to answer the question ``What is the impact on fairness if the training data were sampled from a different group?'' Or we can choose the concept to be the training labels $Y$, and then our method measures the impact on fairness when the label is changed. 

Since influence function can only be defined on a sample, the key challenge is to obtain \textit{counterfactual sample}, i.e. the counterfactual version of a sample as if its concept were changed. For example, if a training sample's sensitive attribute $A$ is changed from female to male, how would the rest of the sample (e.g. feature $X$ and label $Y$) change accordingly? Without considering the dependency between data attributes, we would fail to cover the true nature of training data. In addition, fairness is often influenced by the second- or third-order consequence of the change.\footnote{For example, changing sensitive attribute $A$ of a training sample does not impact the measured fairness because the model is only trained on feature $X$ and label $Y$ (if the training feature $X$ does not include the sensitive attribute $A$), therefore the change on $A$ has to be propagated to the model and the fairness through intermediate $X$ and $Y$.} The problem might sound similar in the causal inference literature \citep{pearl2010causal}. However, it is notoriously hard to approximate counterfactuals in the causal inference, and in our case it is unclear whether the problem is identifiable or not \citep{zhang2009identifiability,zhang2012identifiability,shimizu2006linear,hoyer2008nonlinear}. We, therefore, develop a set of heuristic approximations, notably with W-GAN~\citep{arjovsky2017wasserstein} and \textit{optimal transport mapping}~\citep{villani2009optimal}, to estimate counterfactual samples empirically.

Our flexible framework generalizes the prior works that only consider removing or reweighing training samples~\citep{wang2022understanding,li2022achieving}, and we can provide a broader set of explanations and give more insights to practitioners in a wider scope (\eg what if a data pattern is drawn from another demographic group?). We name our influence framework as \textit{Concept Influence for Fairness} (CIF).

In addition to explaining the cause of unfairness, CIF can also recommend practitioners ways to fix the training data to improve fairness by counterfactually changing concepts in training data. Furthermore, we demonstrate the power of our framework in a number of other applications including (1) detecting mislabeling, (2) detecting poisoning attacks, and (3) fixing imbalanced representation. Through experiments on 4 datasets -- including synthetic, tabular, and image -- we show that our method achieves satisfactory performance in a wide range of tasks.

\section{Influence of Training Concepts}
\label{sec:pred}
We start with introducing the influence function for fairness, the concept in training data, and define our \textit{Concept Influence for Fairness} (CIF).
\subsection{Fairness Influence Function}
We start with introducing the influence function that quantifies the impact of perturbing training samples on the group fairness metrics.

\para{Influence Function on Group Fairness.} Denote the training data by $D_{train} = \{z_i^{tr} = (x_i^{tr}, y_i^{tr})\}_{i=1}^n$ and the validation data by $D_{val} = \{z_i^{val} = (x_i^{val}, y_i^{val})\}_{i=1}^n$.  
Suppose the model is parameterized by $\theta \in \Theta$, and there exists a subset of training data with sample indices $\mathcal K = \{K_1,...,K_k\}$. If we perturb a group $\mathcal K$ by assigning each sample $i \in \mathcal K$ with weight $w_i \in [0, 1]$, denote the resulting counterfactual model's weights by $\hat{\theta}_{\K}$.

\begin{definition} The fairness influence of reweighing group $\mathcal K$ in the training data is defined as the difference of fairness measure between the original model $\hat{\theta}$ (trained on the full training data) and the counterfactual model $\hat{\theta}_{\K}$:
\begin{align}
\infl(D_{val}, \K, \hat{\theta}) := \ell_{\text{fair}}(\hat{\theta})-\ell_{\text{fair}}(\hat{\theta}_{\K})  
\end{align}
\end{definition}
where $\ell_{fair}$ is the fairness measure (will be specified shortly after).

Similar to~\citep{koh2017understanding,koh2019accuracy,li2022achieving}, we can derive the closed-form solution of fairness influence function (see Appendix~\ref{app:infl} for the derivation):

\begin{proposition}
The first-order approximation of $\infl(D_{val}, \K, \hat{\theta}) $ takes the following form:
\begin{equation}
\label{eq:fairinfl}
\infl(D_{val}, \K, \hat{\theta}) \approx -\nabla_{\theta} \ell_{\text{fair}}(\hat{\theta})^{\intercal} H^{-1}_{\hat{\theta}} \left(\sum_{i \in \K} 
 w_i \nabla \ell(z^{tr}_i; \hat{\theta}) \right)
\end{equation}
where $H_{\hat{\theta}}$ is the hessian matrix \ie $H_{\hat{\theta}}:=\frac{1}{n}\nabla^2 \sum_{i=1}^n \ell(z^{tr}_i; \hat{\theta})$, and $\ell$ is the original loss function (e.g. cross-entropy loss in classification).
\end{proposition}

\para{Approximated Fairness Loss.} The loss $\ell_{fair}(\hat{\theta})$ quantifies the fairness of a trained model $\hat{\theta}$. Similarly to prior work~\citep{wang2022understanding,sattigeri2022fair}, we can approximate it with a surrogate loss on the validation data. Denote the corresponding classifier for $\theta$ as $h_\theta$, we can approximate the widely used group fairness Demographic Parity~\citep{calders2009building,chouldechova2017fair} (DP) violation as the following (assume both $A$ and the classification task are binary):
\begin{equation}
\label{eq:fairDP}
\begin{aligned}
    & \ell_{DP}(\hat{\theta}) := \big| \mathbb{P}(h_\theta(X) = 1 | A = 0)-\mathbb{P}(h_\theta(X) = 1 | A = 1) \big| \\
    &\approx \Bigg| \frac{\sum_{i \in D_{val}: a_i = 0} g( z^{val}_i; \theta)}{\sum_{i \in D_{val}}{\mathbb{I}[a_i = 0]}} - \frac{\sum_{i \in D_{val}: a_i = 1} g( z^{val}_i; \theta)}{\sum_{i \in D_{val}}{\mathbb{I}[a_i = 1]}} \Bigg|
\end{aligned}
\end{equation}
where $g$ is the logit of the predicted probability for class $1$. See Appendix~\ref{app:approx} for the approximated violation of Equality of Opportunity~\citep{hardt2016equality} (EOP) and Equality of Odds~\citep{woodworth2017learning} (EO).

\subsection{Concepts in Training Data}

\begin{figure*}
    \centering
    \begin{minipage}{\linewidth}
        \begin{subfigure}[t]{.32\linewidth}
            \includegraphics[width=\textwidth]{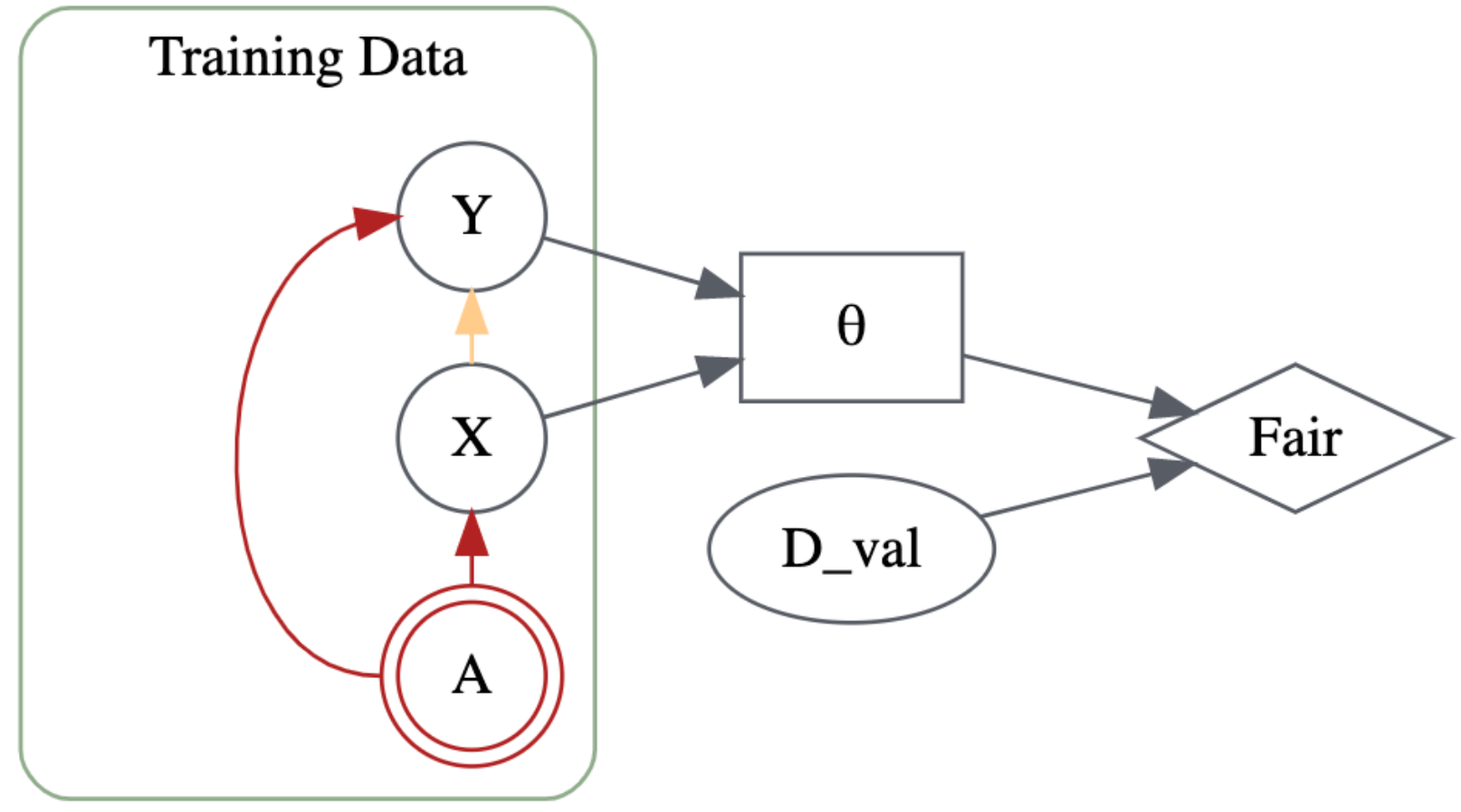}
            \caption{Overriding sensitive attribute $A$.}
            \label{fig:inter_A}
        \end{subfigure} 
        \begin{subfigure}[t]{.32\linewidth}
            \includegraphics[width=\textwidth]{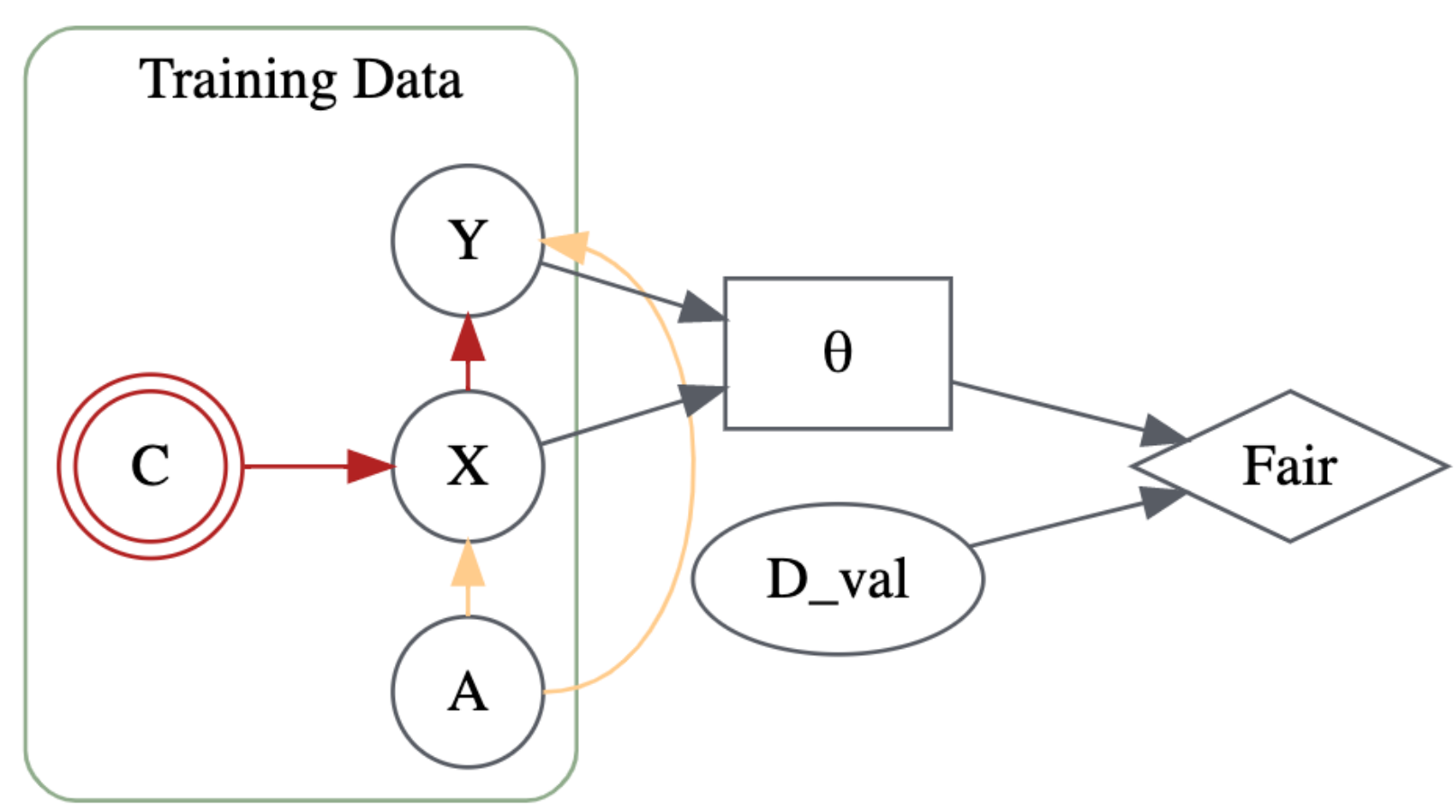}
            \caption{Overriding feature $X$. }
            \label{fig:inter_X}
        \end{subfigure} 
        \begin{subfigure}[t]{.32\linewidth}
            \includegraphics[width=\textwidth]{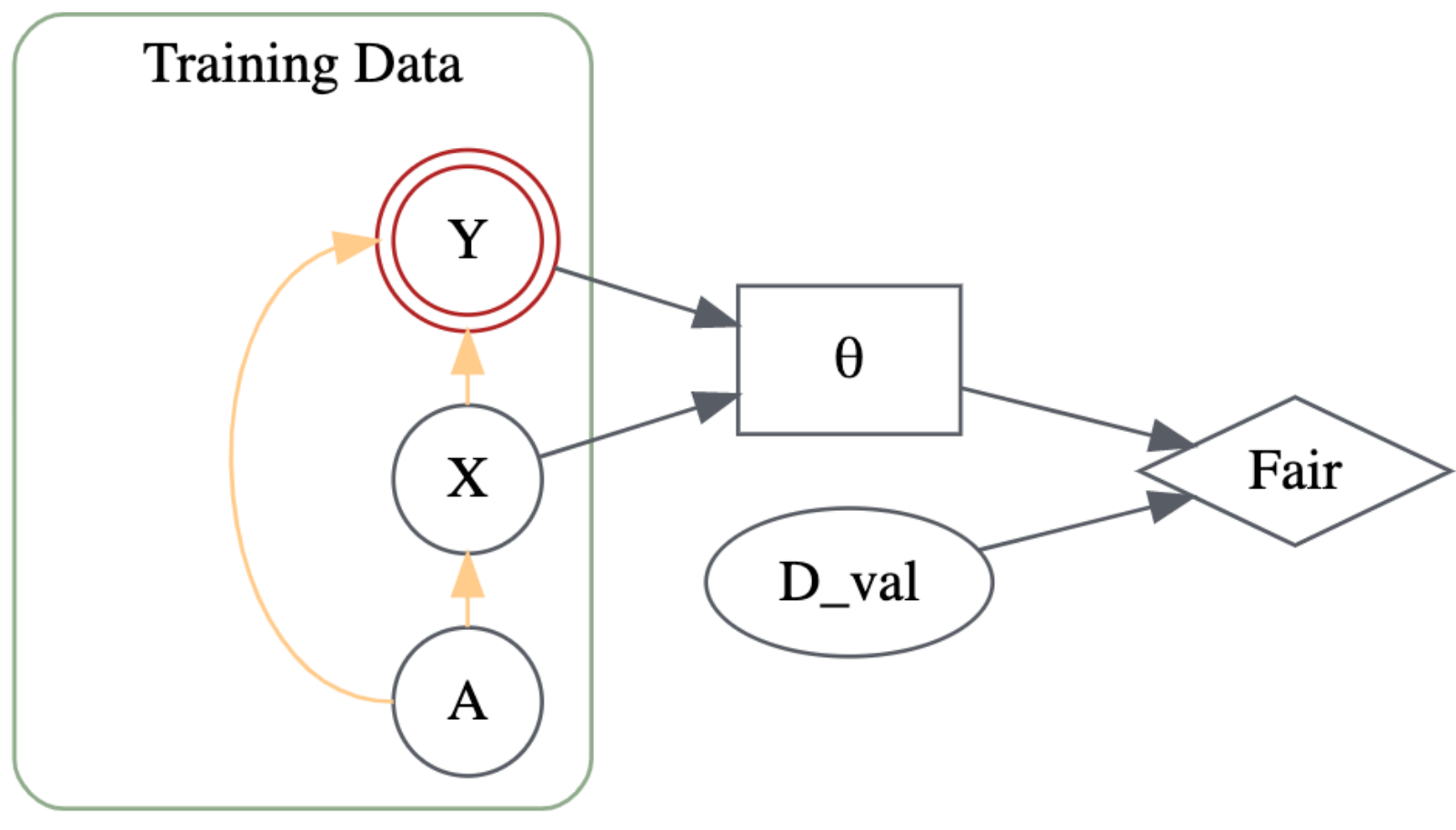}
            \caption{Overriding label $Y$.}
            \label{fig:inter_Y}
        \end{subfigure} 
    \end{minipage}
    \caption{Our data dependency assumption. Yellow arrows represent the data dependency link and red arrows represent the effect of overriding (\ie counterfactually changing the value of a concept). In training data, the concept variable $C$ can override sensitive attribute $A$ (i.e. Figure (a)), features $X$ (i.e. Figure (b)), and label $Y$ (i.e. Figure (c)). We train the model $\theta$ on $X$ and $Y$, and compute the validation fairness metric \texttt{Fair} on the validation dataset $D_{val}$.}
    \label{fig:causal}
\end{figure*}

A concept is a \textit{sample-level} categorical attribute associated with the training data. Formally, denote a concept by $C \in \mathcal C:=\{1,2,...,c\}$ where $C$ is a discrete concept encoded in the dataset $(X,Y,A)$. $C$ can simply be either $Y$ or $A$ or any feature in $X$ or in a broader definition. %
See Figure~\ref{fig:causal} for an illustration of the training concepts and our data dependency assumption. %
Our core idea is to quantify the influence when each training sample is replaced by its \textit{counterfactual sample}, \ie the counterfactual version of the sample if its concept were changed, when we transform the sample w.r.t. a certain concept. We use the term \textit{override} to mean counterfactually changing the concept, \eg overriding a sample's concept to $c$ means counterfactually setting its concept to $c$ and replacing the sample with its counterfactual version as if its concept were $c$. We will formally define overriding in Section~\ref{subsec:cif}.

\para{Examples.} We provide examples of concepts and motivate why transforming samples based on training concepts can be intuitively helpful for fairness.

\squishlist
    \item \para{Concept as Sensitive Attribute ($C=A$).} Intuitively speaking, the sensitive/group attribute relates closely to fairness measures due to its importance in controlling the sampled distribution of each group. Changing $A$ corresponds to asking counterfactually what if a similar or counterfactual sample were from a different sensitive group. %
    \item \para{Concept as Label ($C=Y$).} In many situations, there are uncertainties in the label $Y|X$. Some other times, the observed $Y$ can either encode noise, mislabeling or subjective biases. They can all contribute to unfairness. Changing $Y$ implies the counterfactual effect if we were to change the label (\eg a sampling, a historical decision, or a human judgment) of a sample. 
    \item \para{Concept as Predefined Feature Attribute ($C=attr(X)$).} Our framework allows us to predefine a relevant concept based on feature $X$. $C$ can be either an externally labeled concept (\eg sample-level label in image data) or a part of $X$ (\eg a categorical\footnote{All concepts in $X$, $Y$, or $A$ that we consider are assumed to be categorical because the continuous concept is not well-defined in the literature of concept.} feature in tabular data). For instance, if we want to understand how skin color would affect the model's fairness, and if so which data samples would impact the observed fairness the most w.r.t skin color, we can specify $C = attr(image) \in \{\text{dark}, \text{light}\}$. Then transforming w.r.t. this concept corresponds to identifying samples from different skin colors that, if were included in the training data, would lead to a fairer model.
    \item \para{Concept as Removal.} Our setting is also flexible enough to consider the effect of removing a training sample, as commonly considered in the literature on influence function~\citep{li2022achieving}. Consider a selection variable $S \in \{1,0\}$ for each instance $z^{tr}_i$, for each sample that appears in the training data we have $s_i = 1$. Changing to $s_i = 0$ means the sample is counterfactually removed, \ie $\hat{z}^{tr}_{i}(c') = \varnothing$. By allowing the removal concept, we can incorporate the prior works on the influence of removing samples into our framework.
\squishend

\subsection{Concept Influence for Fairness (CIF)}
\label{subsec:cif}

Our goal is to quantify the counterfactual effect of changing concept $c$ for each data sample $(x,y,a)$. %
Mathematically, denote by $(\hat{x}, \hat{y}, \hat{a})$ the counterfactual sample by overriding $c$. Consider a training sample $z^{tr}_{i} := (x_i,y_i,a_i,c_i)$, and define a counterfactual sample for $z^{tr}_{i}$ when counterfactually changing from $C=c$ to $C=c'$ as follows:
\begin{equation}
    \hat{x}(c'), \hat{y}(c'), \hat{a}(c') = \textit{transform}(
    X=x, Y=y, A=a, \textit{override}(C= c')), ~c' \neq c. %
\end{equation}
In the definition above, $\textit{override}(\cdot)$ operator counterfactually sets the value of the concept variable to a different $c'$. If differs from merely $C = c$ in the sense that the change on $C$ would also change other variables, \ie $A$, $X$, $Y$, or the action of removing samples. It also differs from the well-known \textit{do}-operator in the causal literature \citep{pearl2010causal}, in the sense that the procedure does not necessarily need to follow the mechanism of causal inference, and therefore can include any \textit{empirical} mechanisms that approximate counterfactuals. The difference is vital because when it is unclear whether the problem is identifiable or not \citep{zhang2009identifiability,zhang2012identifiability,shimizu2006linear,hoyer2008nonlinear}, we still want to develop heuristic approximations. And $\textit{transform}(\cdot)$ maps the original training samples to their corresponding counterfactual samples by considering the effect of $\textit{override}(\cdot)$.

Therefore, our $\textit{override}(\cdot)$ and $\textit{transform}(\cdot)$ can include both traditional causal inference methods when we deal with synthetic data, and, more importantly, empirical heuristics when identifiability is unclear. We include three typical scenarios:
\squishlist
    \item Assigning values: When we override label $Y$, we can simply set the label value (which is not a typical case in causal inference), \ie
\begin{equation}
\textit{transform}(X=x, Y=y, A=a, \textit{override}(Y= y'))
= (x, y', a)
\end{equation}
 \item Empirical approximation: When identifiability is unclear, \textit{which is the major case that we study}, we can approximate the counterfactual samples by training a generative model $G$,\footnote{In Section \ref{sec:cs}, we introduce how to construct $\textit{transform}(\cdot)$ in this case.} \ie
\begin{equation}
\textit{transform}(X=x, Y=y, A=a,\textit{override}(C= c'))
= \text{G}_{c \rightarrow c'} (x, y, a)
\end{equation}
    \item \textit{Do}-intervention: When the counterfactual distribution is theoretically identifiable, which is only in the synthetic setting, then the $\textit{transform}(\cdot)$ and $\textit{override}(\cdot)$ are the sampling functions and the \textit{do}-operator, \ie\footnote{The definition is slightly abused -- when $C$ overlaps with any of $(X,Y,A)$, the $do(\cdot)$ operation has a higher priority and is assumed to automatically overwrite the other dependencies. For example, when $C=A$, we have:
$\mathbb{P}\left(\hat{X}, \hat{Y}, \hat{A}|X=x, Y=y, A=a, do(C=c')\right) =   \mathbb{P}\left(\hat{X}, \hat{Y}, \hat{A}|X=x, Y=y, \hcancel[red]{A=a}, do(A=\hat{a})\right)$.}
\begin{equation}
    \begin{aligned}
    &\textit{transform}(X=x, Y=y, A=a, \textit{override}(C = c')) \\
    &= \hat{x}(c'), \hat{y}(c'), \hat{a}(c')\sim\mathbb{P}\left(\hat{X}, \hat{Y}, \hat{A}|X=x, Y=y, A=a, do(C=c')\right), ~c' \neq c.
    \end{aligned}
\end{equation}

\squishend

The effectiveness of our solution depends on finding a proper $\textit{transform}(\cdot)$, which is our work's focus, and the quality of $\textit{transform}(\cdot)$ can be verified empirically in experiments. In addition, our framework is general; if researchers discover better ways to approximate counterfactuals, they plug those into our framework easily.

\para{Concept Influence for Fairness.} Denote a counterfactual sample as 
$
    \hat{z}^{tr}_{i}(c') = (\hat{x}_i(c'), \hat{y}_i(c'), \hat{a}_i(c'), \hat{c}_i = c')
$. Then we define the counterfactual model when replacing $z^{tr}_{i} = (x_i,y_i,a_i,c_i)$ with $\hat{z}^{tr}_{i}(c') $ as:
\begin{align}
 \hat{\theta}_{i,c'} := \text{argmin}_{\theta} \{R(\theta) - \epsilon \cdot \ell(\theta, z^{tr}_{i}) +  \epsilon \cdot  \ell(\theta, \hat{z}^{tr}_{i}(c'))\}  
\end{align}

\begin{definition}[Concept Influence for Fairness (CIF)] The concept influence for fairness (CIF) of overriding on a concept $C$ to $c'$ in sample $i$ on the fairness loss $\ell_{\text{fair}}$ is defined as: 
\begin{align}
 \infl(D_{val}, \hat{\theta}_{i,c'}) := \ell_{\text{fair}}(\hat{\theta})  - \ell_{\text{fair}}(\hat{\theta}_{i,c'})   
\end{align}
\end{definition}
Based on Proposition \ref{eq:fairinfl}, we can easily prove (see Appendix~\ref{app:proof_fairinfl} for the proof):

\begin{proposition}
The concept influence for fairness (CIF) of a training sample $z^{tr}_{i}$ when counterfactually transformed to $\hat{z}^{tr}_{i}(c')$ based on the target concept $c'$ can be computed as:
\begin{equation}
\label{eq:infl_single}
\infl(D_{val}, \hat{\theta}_{i,c'}) \approx - \nabla_{\theta} \ell_{\text{fair}}(\hat{\theta})^{\intercal} H^{-1}_{\hat{\theta}} \left( \nabla \ell(z^{tr}_{i}; \hat{\theta}) - \nabla \ell(\hat{z}^{tr}_{i}(c'); \hat{\theta})  \right)
\end{equation}
\end{proposition}

\para{Why Can CIF Improve Fairness?} We include the full theoretical analysis of why overriding training concepts using CIF framework can improve fairness in Appendix~\ref{sec:why}. We briefly summarize here. When overriding label $Y$, we can change a training label of a disadvantaged group from a wrong label to a correct one, and effectively improve the performance of the model for this group. Therefore the label (re)assignment can reduce the accuracy disparities. Overriding sensitive attribute $A$ improves fairness by balancing the data distribution. Later in the experiments (Figure~\ref{fig:resample} in Appendix~\ref{app:abl}), we show that the influence function often identifies the data from the majority group and recommends them to be changed to the minority group.

\section{Method}
\label{sec:algo}

We present our method of generating counterfactual samples and computing CIF.

\subsection{Generating Counterfactual Samples} 
\label{sec:cs}

To compute the fairness influence based on Eqn.~\ref{eq:infl_single}, we need to first generate the corresponding counterfactual sample $\hat{z}^{tr}_{i}(c') = (\hat{x}_i(c'), \hat{y}_i(c'), \hat{a}_i(c'), \hat{c}_i = c')$ when we override concept $C$ from $c$ to $c'$. Theoretically, generating the counterfactual examples requires the assumptions of the underlying causal graph but we use a set of practical algorithms to approximate.

\para{Overriding Label $Y$.} Since there is no variable in training data dependent on $Y$ (Figure~\ref{fig:inter_Y}), we can simply change the sample's label to the target label $\hat{y}_i$ and keep other attributes unchanged, \ie $\hat{z}^{tr}_{i}(\hat{y}_i) = (x_i, \hat{y}_i, a_i,  \hat{c}_i = \hat{y}_i)$.

\para{Overriding Sensitive Attribute $A$.} When we override a sample's $A$, both its $X$ and $Y$ need to change (Figure~\ref{fig:inter_A}). This is the same as asking, \eg in a loan application, ``How a female applicant's profile (\ie $x_i$) and the loan decision (\ie $y_i$) would change, had she been a male (\ie $a_i = \hat{a_i}$)?'' Inspired by~\citep{black2020fliptest}, we train a W-GAN~\citep{arjovsky2017wasserstein} with \textit{optimal transport mapping}~\citep{villani2009optimal} to generate \textit{in-distributional}\footnote{We need the counterfactual samples to be in-distributional rather than out-of-distributional because we need the change between the counterfactual sample and the original sample to be large enough to impact the fairness measure.
} counterfactual samples for $x_i$ as if $x_i$ belongs to a different $a_i$. To do so, we need to map the distribution of $X$ from $A = a$ to $A = a'$. We first partition the training samples' feature into two groups: $X | A = a$ and $X | A = a'$. Then we train a W-GAN with the generator $G_{a \rightarrow a'}$ as the approximated optimal transport mapping from $X | A = a$ to $X | A = a'$ and the discriminator $D_{a \rightarrow a'}$ ensures the mapped samples $G_{a \rightarrow a'}(X)$ and the real samples $X | A = a'$ are indistinguishable. The training objectives are the following:
\begin{equation}
\begin{aligned}
&\ell_{G_{a \rightarrow a'}} = \frac{1}{n}\Big(\sum_{x \in X|A=a}{D(G(x)) +\lambda \cdot \sum_{x \in X|A=a}{c(x, G(x))}} \Big)\\
&\ell_{D_{a \rightarrow a'}}=\frac{1}{n} \Big(\sum_{x'\in X|A=a'}{D(x') - \sum_{x \in X|A=a}{D(G(x))}\Big)}
\end{aligned}
\end{equation}
where $n$ is the number of training samples, $\lambda$ is the weight balancing the conventional W-GAN generator loss (\ie the first term in $\ell_{G_{a \rightarrow a'}}$) and the distance cost function $c(.)$ (\ie $\ell_2$ norm in our case) that makes sure the mapped samples are not too far from the original distribution. 

After we train the W-GAN on the training data, we can use the trained generator $G_{a \rightarrow a'}$ to map a sample $x_i$ to its counterfactual version $\hat{x}_i = G_{a_i \rightarrow \hat{a}_i}(x_i)$. In addition, once we have the counterfactual features, we can use the original model to predict the corresponding counterfactual label (\ie following the dependency link $X \rightarrow Y$ in Figure~\ref{fig:inter_A}). The resulting counterfactual sample is $\hat{z}^{tr}_{i}(\hat{a}_i) = (\hat{x}_i, h_{\hat{\theta}}(\hat{x}_i), \hat{a}_i,\hat{c}_i = \hat{a}_i)$.

\para{Overriding Feature $X$.} In image data, assume there exists an image-label attribute $C = attr(X)$, \eg young or old in facial images, and overriding $X$ means transforming the image (\ie all pixel values in $X$) as if it belongs to a different $C$. In tabular data, $C$ is one of the features in $X$, and when $C$ is changed, all other features in $X$ need to change accordingly. In both cases, similar to overriding $A$, we train a W-GAN to learn the mapping from the group $X|C = c$ to $X|C = c'$; the resulting generator is $G_{c \rightarrow c'}$ and the generated counterfactual feature is $\hat{x}_i = G_{c_i \rightarrow \hat{c}_i}(x_i)$. Similarly, since the data dependency $X \rightarrow Y$ exists in our assumption in Figure~\ref{fig:inter_X}, we also use the original model's predicted label as the counterfactual label. The resulting counterfactual sample is $\hat{z}^{tr}_{i}(\hat{c}_i) = (\hat{x_i}, h_{\hat{\theta}}(\hat{x_i}), a_i, \hat{c}_i = \hat{x}_i)$.

\para{Removal.} Removing is simply setting the counterfactual sample to be null, \ie $\hat{z}^{tr}_{i}(c') = \varnothing$.

\subsection{Computing Influence} Following~\citep{koh2017understanding}, we use the Hessian vector product (HVP) to compute the product of the second and the third term in Eqn.~\ref{eq:infl_single} together. Let $v := \left( \nabla \ell(z^{tr}_{i}; \hat{\theta}) - \nabla \ell(\hat{z}^{tr}_{i}(c'); \hat{\theta})  \right)$, we can compute $H^{-1}v$ recursively~\citep{agarwal2017second}:
\begin{equation}
    \hat{H}_r^{-1}v = v + (I - \hat{H}_0)\hat{H}_{r-1}^{-1}v
\end{equation}
where $\hat{H}_0$ is the Hessian matrix approximated on random batches. Let $t$ be the final recursive iteration, then the final CIF is $\infl(D_{val}, \hat{\theta}_{i,c'}) \approx - \nabla_{\theta} \ell_{\text{fair}}(\hat{\theta})^{\intercal}\hat{H}_t^{-1}v$, where $\ell_{\text{fair}}(\hat{\theta})$ is the surrogate loss of fairness measure (\eg Eqn.~\ref{eq:fairDP}, ~\ref{eq:fairEOP} or ~\ref{eq:fairEO}). 

Similar to~\citep{koh2017understanding}, we assume the loss is twice-differentiable and strongly convex in $\theta$, so that $H_{\hat{\theta}}$ exists and is positive definite, \ie $H^{-1}_{\hat{\theta}}$ exists.  If the assumptions are not satisfied, the convergence would suffer. It is a well-documented problem in the literature~\citep{koh2017understanding,basu2020influence}. However, our goal is not to propose a better influence function approximating algorithm; we aim to demonstrate the idea of leveraging influence function to help practitioners understand the unfairness. As better influence-approximating algorithms are invented, our framework is flexible enough to plug in and benefit from the improvement. Later in Section~\ref{subsec:miti}, we empirically show the accuracy of our influence estimation.

\section{Experiments}
\label{sec:eval}

\begin{figure*}[t]
  \begin{minipage}[t]{\linewidth}
    \centering
  \includegraphics[width=\linewidth]{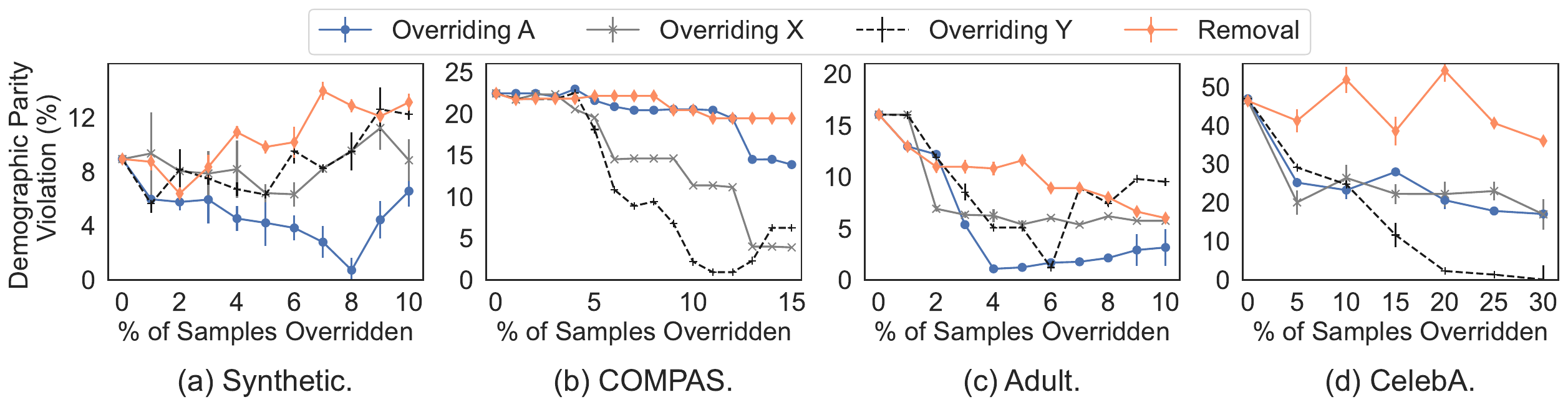}
  \caption{CIF-based mitigation performance with fairness measure Demographic Parity (DP).}
  \label{fig:miti_DP}
  \end{minipage}
  \begin{minipage}[t]{\linewidth}
    \centering
    \includegraphics[width=\linewidth]{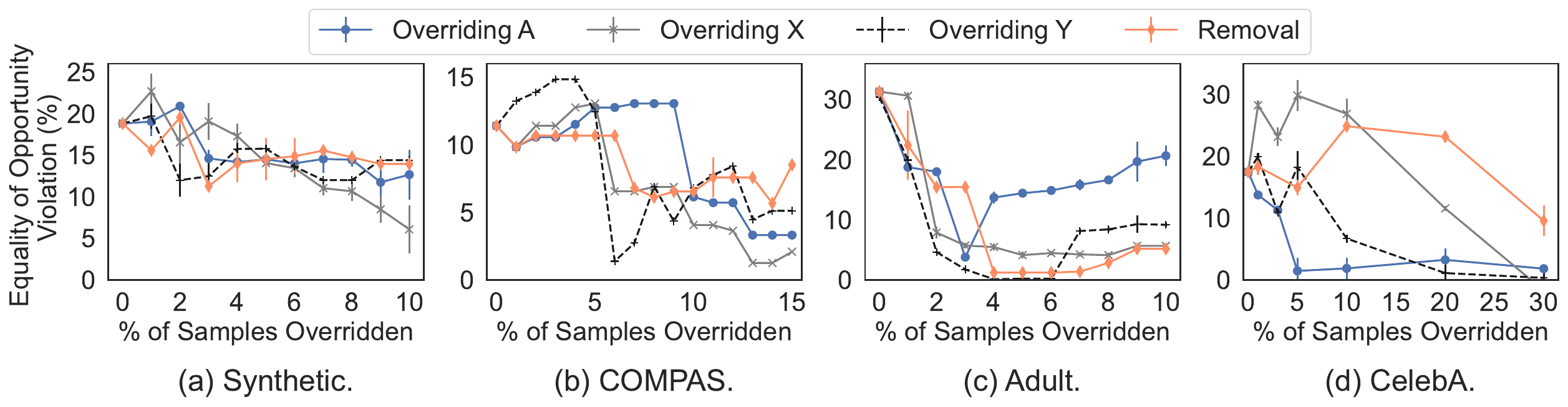}
    \caption{CIF-based mitigation performance with fairness measure Equality of Opportunity (EOP).}
    \label{fig:miti_eop}
  \end{minipage}
    \begin{minipage}[t]{\linewidth}
    \centering
    \includegraphics[width=\linewidth]{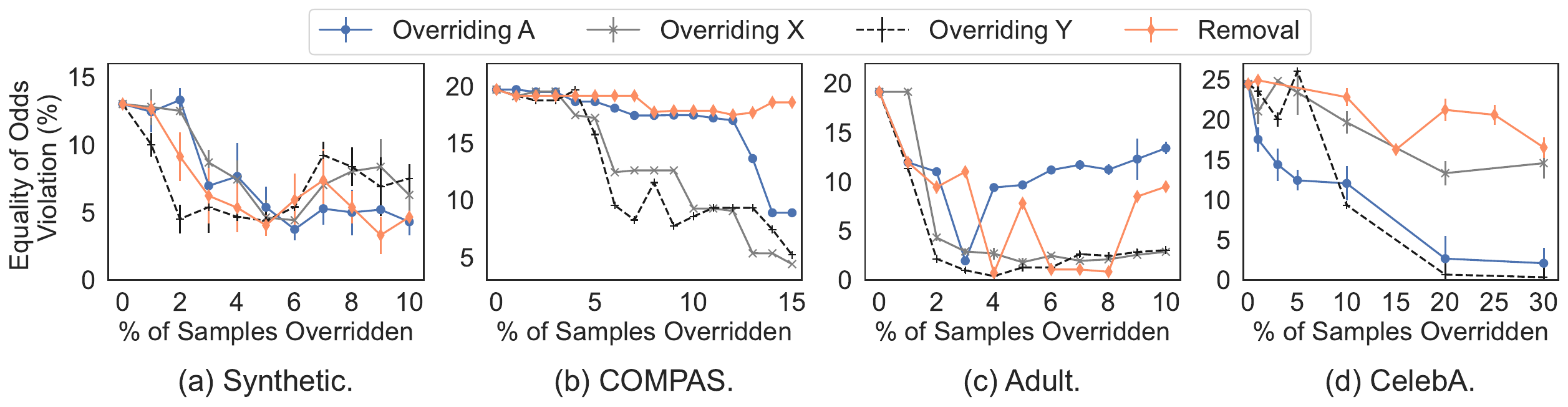}
    \caption{CIF-based mitigation performance with fairness measure Equality of Odds (EO).}
    \label{fig:miti_eo}
  \end{minipage}
\end{figure*}

We present a series of experiments to validate the effectiveness of CIF in explaining and mitigating model unfairness, detecting biased/poisoned samples, and recommending resampling to balance representation.

We test CIF on 4 datasets: synthetic, COMPAS~\citep{angwin2016machine}, Adult~\citep{kohavi1996scaling}, and CelebA~\citep{liu2015faceattributes}. We report results on three group fairness metrics (DP, EOP, and EO, see Table~\ref{tab:fairness_approx} in Appendix~\ref{app:approx} for the definition). We include dataset details in Appendix~\ref{app:dataset} and experiment details in Appendix~\ref{app:exp_details}.

\subsection{Mitigation Performance}
\label{subsec:miti}

We test the CIF-based mitigation by first computing CIF values on all training samples, and then replacing samples with the highest CIF values by their corresponding generated counterfactual samples, and retraining the model. In each retraining on the removed training set, we repeat the training process 10 times and report the standard deviation of the fairness measure in the error bars. Figure~\ref{fig:miti_DP}-\ref{fig:miti_eo} show the fairness performance after the model training. We observe that all three fairness measures improve significantly after following CIF's mitigation recommendations. See Figure~\ref{fig:miti_acc_DP}-\ref{fig:miti_acc_eo} in Appendix~\ref{app:mitigation} for the reported model accuracy.

We summarize observations: (1) Overriding $Y$ is highly effective on real-world data but not on synthetic. We conjecture that this is because we control the synthetic data to be cleanly labeled, which is not the case for other real-world data. To test this hypothesis, we add label noise in the synthetic data to see if it would make Y-overriding more effective. Table~\ref{tab:syn_label} in Appendix~\ref{app:label_noise} shows the results. When labels are no longer clean, our Y-overriding becomes more effective, showing that indeed label noise can be a significant contributor to the unfairness, as also indicated by prior work~\citep{wang2021fair}.

(2) Overriding $A$ proves to be helpful for most cases, especially for DP, which highly relates to the demographic variable $A$.

(3) We set the size of synthetic data to be small (1,000) to show that simply removing training samples might not always be a good strategy, particularly on a small dataset in which the model would suffer significantly from losing training samples.

\para{Fairness-utility Tradeoff.} We report the fairness-utility tradeoffs of our mitigation on COMPAS, together with the in-processing mitigation~\citep{agarwal2018reductions} in Figure~\ref{fig:tradeoff}. Our mitigation is comparable to \citep{agarwal2018reductions}; sometimes we can achieve better fairness given a similar level of accuracy (\eg when accuracy is $\sim 60\%$). 

\para{Accuracy of CIF Estimate.} Figure~\ref{fig:infl_est_compass} plots influence value vs. the actual difference in fairness loss (DP) on the COMPAS dataset. See Appendix~\ref{app:infl_val} for experiment details. The relationship between our estimated influence and the actual change in fairness is largely linear, meaning our influence value can estimate the fairness change reasonably well.

\para{Distribution of CIF.} We show the distribution of influence values computed on COMPAS corresponding to three fairness metrics in Figure~\ref{fig:infl_dist} (Appendix~\ref{app:infl_dist}). Overriding $Y$ has the highest influence value. This is because we change the value of $Y$ directly in this operation, which is more ``unnaturally'' compared to generating more ``natural'' counterfactual examples with W-GAN (overriding $X$ and $A$) or model-predicted value of $Y$ (overriding $X$). So practitioners should be particularly cautious about mislabelling, \eg if any unprivileged group should be labeled favorable but ended up getting labeled unfavorable.

\subsection{Additional Applications of CIF}
We provide three examples of additional applications that can be derived from our CIF framework: (1) fixing mislabelling, (2) defending against poisoning attacks, and (3) resampling imbalanced representations. We include experiment details and results in Appendix~\ref{app:abl}.

\para{Fixing Mislabeling.} We flip training labels $Y$ in the Adult dataset to artificially increase the model's unfairness. Following~\citep{wang2021fair}, we add group-dependent label noise, \ie the probability of flipping a sample's $Y$ is based on its $A$, to enlarge the fairness gap. We then compute $Y$-overriding CIF on each sample, and flag samples with the top CIF value. In Figure~\ref{fig:abl_flip_y}, we report the precision of our CIF-based detection and mitigation performance if we flip the detected samples' labels and retrain the model. Our detection can flag the incorrect labels that are known to be the source of the unfairness with high precision (compared to randomly flagging the same percentage) and improves the model fairness if the detected labels are corrected.

\begin{figure*}[t]
\begin{minipage}[b]{0.39\linewidth}
    \centering
  \includegraphics[width=\linewidth]{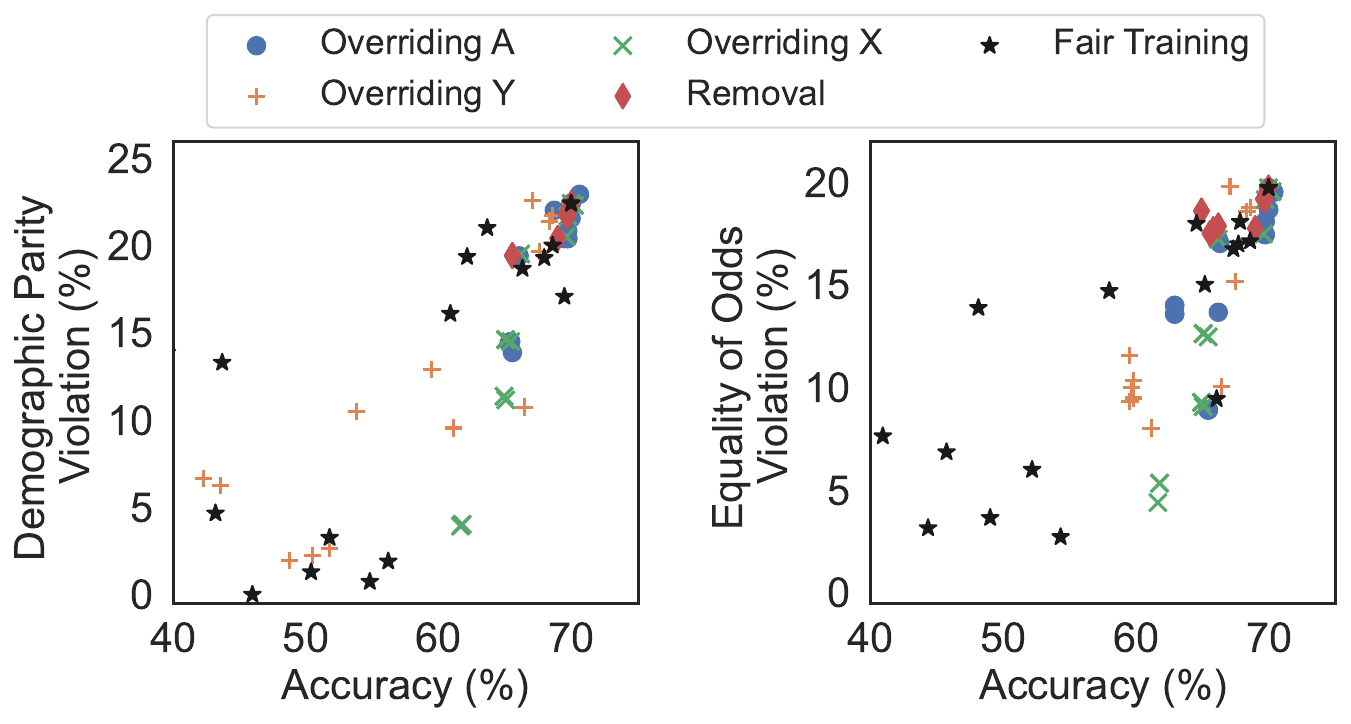}
  \caption{Fairness-accuracy tradeoff of CIF-based mitigation on COMPAS. CIF-based mitigation is comparable to in-processing mitigation method, and sometimes achieves better fairness given a similar level of accuracy.}
  \label{fig:tradeoff}
  \end{minipage}
    \hfill
  \begin{minipage}[b]{0.58\linewidth}
    \centering
  \includegraphics[width=\linewidth]{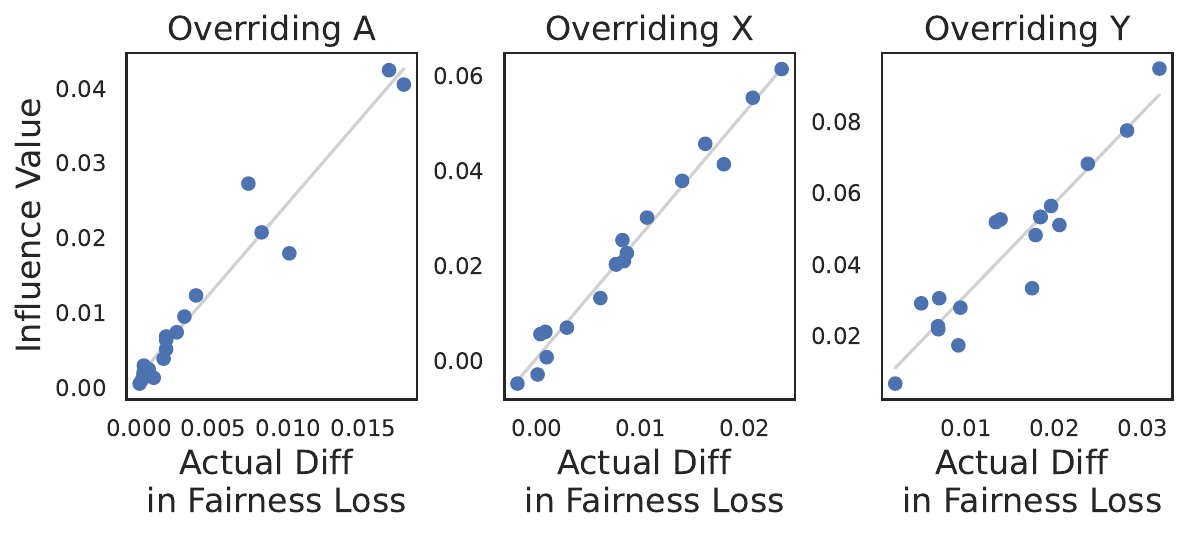}
  \caption{Estimated influence value vs. the actual difference in fairness loss on COMPAS with fairness metric Demographic Parity (DP).}
  \label{fig:infl_est_compass}
  \end{minipage}
\end{figure*}

\para{Defending against Poisoning Attacks.} We demonstrate the application of defending models against fairness poisoning attacks. To generate poisoned training samples that cause the model's unfairness, we choose poisoned training samples with the same probability based on the group- and label-dependent probability in the previous application. In addition to flipping the samples' labels, we also set the target feature (\ie race in Adult) to be a fixed value (\ie white) regardless of the original feature value. The attack that modifies a sample's feature to be a fixed value and changes its label is known as backdoor attack~\citep{gu2019badnets,li2021anti,wu2022backdoorbench}, a special type of poisoning attack. After the poisoning, all fairness measures become worse. For detection, we compute $X$-overriding CIF on the poisoned feature, and flag samples with high CIF value. For mitigation, if we flag a sample to be poisoned, we remove it from the training set and retrain the model. Figure~\ref{fig:poison} shows the precision of our detection and the mitigation performance after removal. We observe a high precision and reasonably good fairness improvement.

\begin{figure}[t]
\centering
  \includegraphics[width=0.5\linewidth]{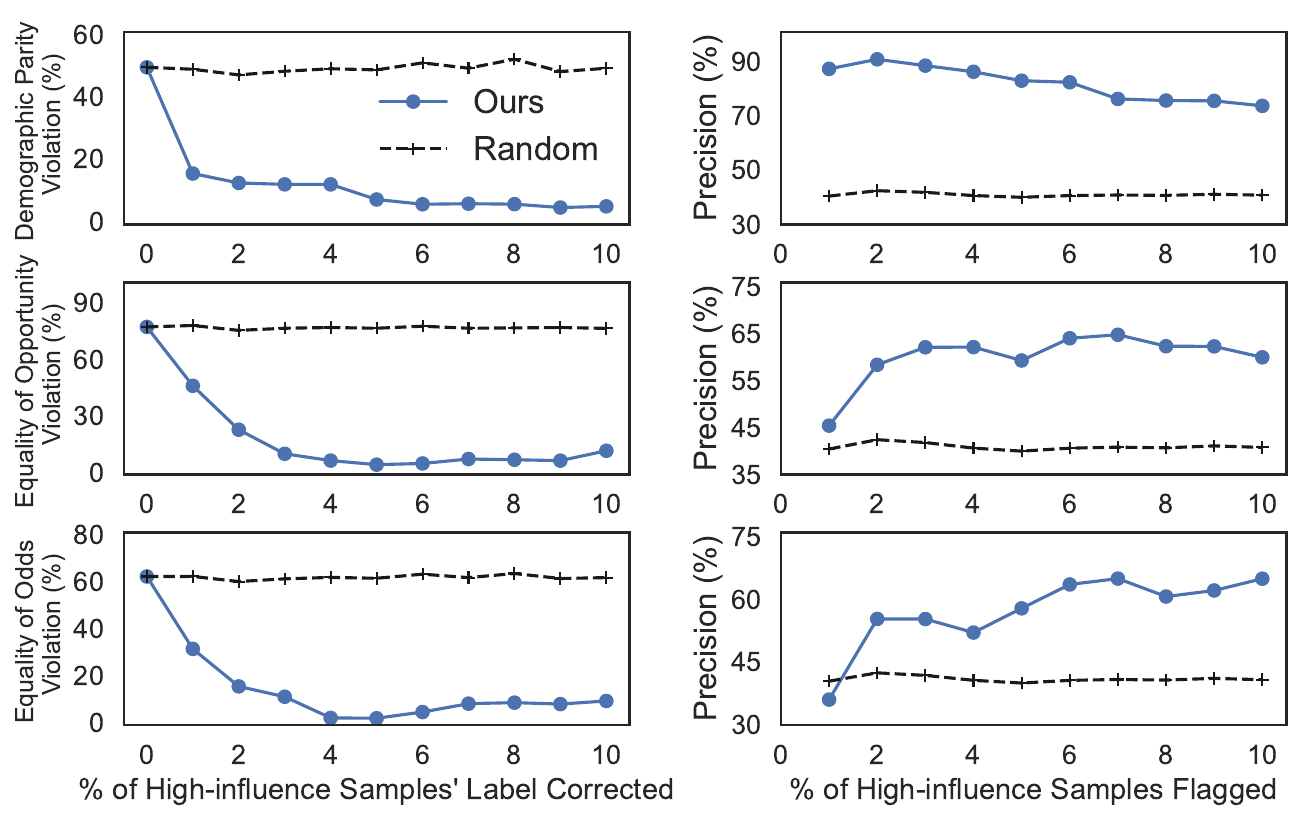}
  \caption{Precision and mitigation performance of using $Y$-overriding CIF to detect and correct training mislabeling that causes bias on Adult.}
  \label{fig:abl_flip_y}
\end{figure}

\para{Resampling Imbalanced Representations.} To create an extremely imbalanced representation in the training set, we upsample the positive samples in the privileged group (\ie male) by $200\%$ in the Adult dataset, 
further increasing the percentage of positive samples that belong to the privileged group,
and therefore the training samples are overwhelmingly represented by the privileged group. The resulting fairness becomes worse. We then compute $A$-overriding CIF, and replace the high-influence samples with their counterfactual samples (\ie adding counterfactual samples in the unprivileged group and reducing samples from the privileged group). In Figure~\ref{fig:resample}, we report the percentage of high-influence samples that belong to the privileged group (\ie how much CIF recommends the data balancing) and the mitigation performance. The high-influence samples are almost all from the privileged group, which is expected, and if they were converted to the counterfactual samples as if they are from the unprivileged group, \ie recollecting and resampling the training distribution, then fairness can improve.

\begin{figure}[t]
  \begin{minipage}[b]{0.46\linewidth}
    \centering
  \includegraphics[width=\linewidth]{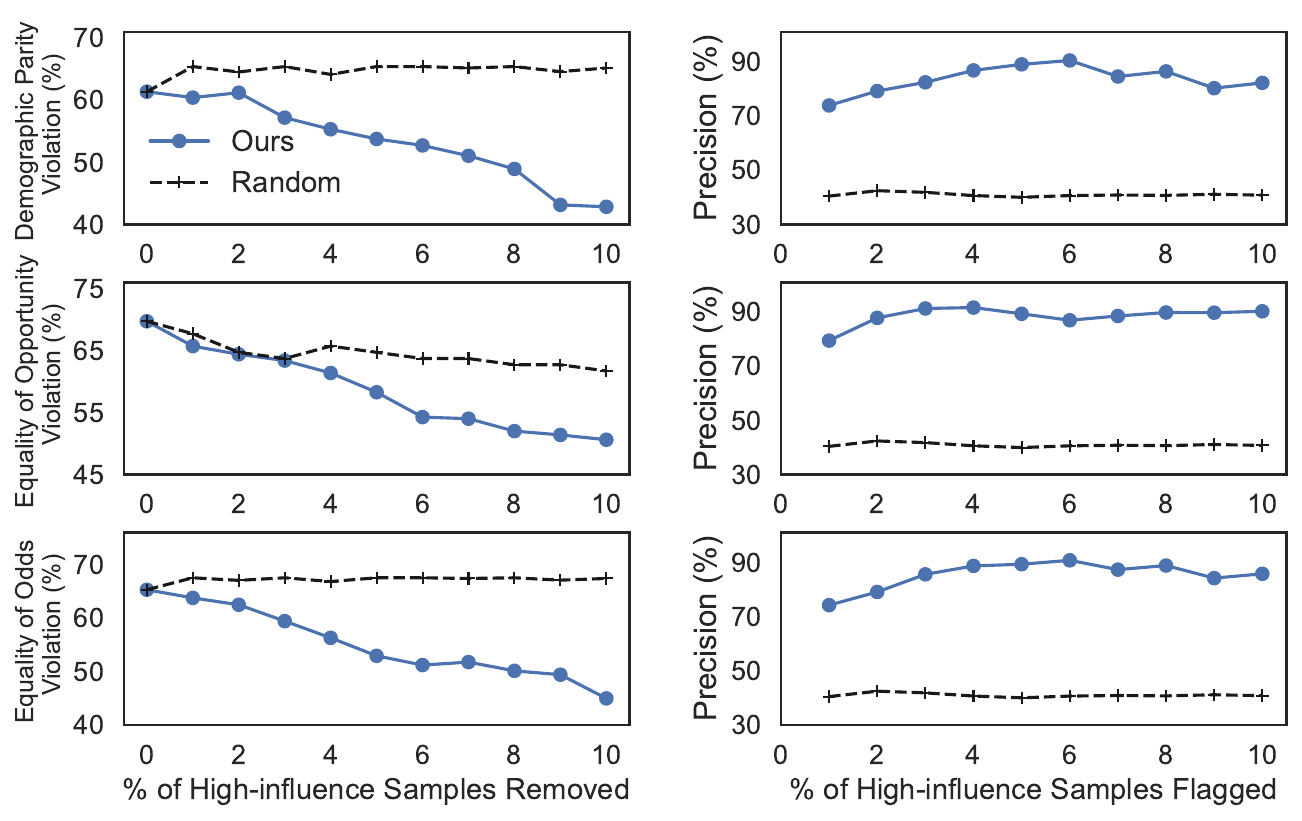}
  \caption{Precision and mitigation performance of using $X$-overriding CIF to detect and correct poisoned training samples that cause unfairness on Adult.}
  \label{fig:poison}
  \end{minipage}
  \hfill
\begin{minipage}[b]{0.49\linewidth}
    \centering
  \includegraphics[width=\linewidth]{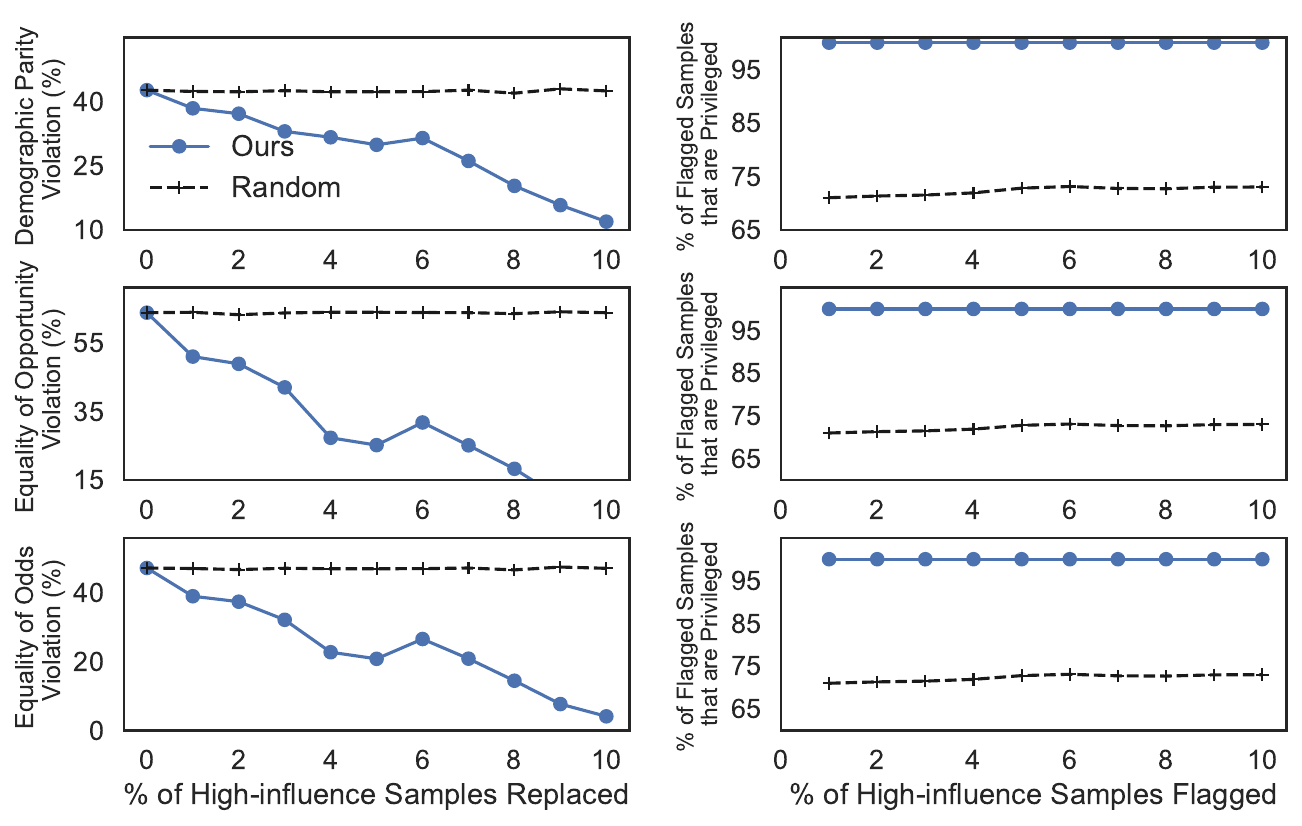}
  \caption{Performance of using $A$-overriding CIF to detect and correct imbalanced training representation that causes unfairness on Adult.}
  \label{fig:resample}
  \end{minipage}
\end{figure}

\section{Related Work}
\label{sec:related}

\para{Influence Function.} The goal of influence function is to quantify the impact of training data on the model's output. \citep{koh2017understanding} popularizes the idea of training data influence to the attention of our research community and has demonstrated its power in a variety of applications. Later works have aimed to improve the efficiency of computing influence functions. For example, Tracein \citep{pruthi2020estimating} proposes a first-order solution that leverages the training gradients of the samples, and a neural tangent kernel approach for speeding up this task. Other works have explored the computation of group influence \citep{basu2020second}, the robustness of influence function \citep{basu2020influence}, its application in explainable AI \citep{linardatos2020explainable} and other tasks like graph networks~\citep{chen2022characterizing}.

\para{Influence Function for Fairness.} Our work is closely relevant to the recent discussions on quantifying training data's influence on a model's fairness properties. \citep{wang2022understanding} computes the training data influence to fairness when removing a certain set of training samples. \citep{li2022achieving} discusses a soft version of the removal and computes also the optimal ``removal weights" for each sample to improve fairness. And \citep{sattigeri2022fair} leverages the computed influence to perform a post-hoc model update to improve its fairness. Note that those works consider the fairness effect of removing or reweighing training samples. Our work targets a more flexible and powerful definition of influence that can give practitioners a wider scope of understanding by introducing the idea of concepts and generating counterfactual samples as well as result in a wider range of potential applications.

\para{Data Repairing for Fairness.} Our work is also related to the work on data repairing to improve fairness. \citep{krasanakis2018adaptive,lahoti2020fairness} discuss the possibilities of reweighing training data to improve fairness. \citep{zhang2022fairness} proposes a ``reprogramming" framework that modified the features of training data. \citep{liu2021can} explores the possibility of resampling labels to improve the fairness of training. Other works study the robustness of model w.r.t fairness~\citep{wang2022fairness,chhabra2022robust,li2022learning}. Another line of research that repairs training data is through training data pre-proccessing~\citep{calmon2017optimized,celis2020data,kamiran2012data,du2018data}, synthetic fair data~\citep{sattigeri2019fairness,jang2021constructing,xu2018fairgan,van2021decaf}, and data augmentation~\citep{sharma2020data,chuang2021fair}.

\section{Conclusions and Limitations}
We propose \textit{Concept Influence for Fairness} (CIF), which generalizes the definition of influence function for fairness from focusing only on the effects of removing or reweighing the training samples to a broader range of dimensions related to the training data's properties. The main idea is to consider the effects of transforming the sample based on a certain \textit{concept} of training data, which is a more flexible framework to help practitioners better understand unfairness with a wider scope and leads to more potential downstream applications. 

We point out two limitations: (1) CIF needs to generate counterfactual samples w.r.t different concepts, which can be computationally expensive and (2) in CIF-based mitigation, it can be non-trivial to determine the optimal number of training samples to override that would maximally improve fairness.

\nocite{langley00}

\bibliography{references}

\begin{thebibliography}{75}
\providecommand{\natexlab}[1]{#1}
\providecommand{\url}[1]{\texttt{#1}}
\expandafter\ifx\csname urlstyle\endcsname\relax
  \providecommand{\doi}[1]{doi: #1}\else
  \providecommand{\doi}{doi: \begingroup \urlstyle{rm}\Url}\fi

\bibitem[Adel et~al.(2019)Adel, Valera, Ghahramani, and Weller]{adel2019one}
T.~Adel, I.~Valera, Z.~Ghahramani, and A.~Weller.
\newblock One-network adversarial fairness.
\newblock In \emph{Proc. of AAAI}, 2019.

\bibitem[Agarwal et~al.(2018)Agarwal, Beygelzimer, Dud{\'\i}k, Langford, and Wallach]{agarwal2018reductions}
A.~Agarwal, A.~Beygelzimer, M.~Dud{\'\i}k, J.~Langford, and H.~Wallach.
\newblock A reductions approach to fair classification.
\newblock In \emph{Proc. of ICML}, 2018.

\bibitem[Agarwal et~al.(2017)Agarwal, Bullins, and Hazan]{agarwal2017second}
N.~Agarwal, B.~Bullins, and E.~Hazan.
\newblock Second-order stochastic optimization for machine learning in linear time.
\newblock \emph{The Journal of Machine Learning Research}, 18\penalty0 (1):\penalty0 4148--4187, 2017.

\bibitem[Angwin et~al.(2016)Angwin, Larson, Mattu, and Kirchner]{angwin2016machine}
J.~Angwin, J.~Larson, S.~Mattu, and L.~Kirchner.
\newblock Machine bias.
\newblock In \emph{Ethics of Data and Analytics}, pages 254--264. Auerbach Publications, 2016.

\bibitem[Arjovsky et~al.(2017)Arjovsky, Chintala, and Bottou]{arjovsky2017wasserstein}
M.~Arjovsky, S.~Chintala, and L.~Bottou.
\newblock Wasserstein generative adversarial networks.
\newblock In \emph{Proc. of ICML}, 2017.

\bibitem[Bagdasaryan et~al.(2019)Bagdasaryan, Poursaeed, and Shmatikov]{bagdasaryan2019differential}
E.~Bagdasaryan, O.~Poursaeed, and V.~Shmatikov.
\newblock Differential privacy has disparate impact on model accuracy.
\newblock \emph{Advances in neural information processing systems}, 32, 2019.

\bibitem[Basu et~al.(2020{\natexlab{a}})Basu, Pope, and Feizi]{basu2020influence}
S.~Basu, P.~Pope, and S.~Feizi.
\newblock Influence functions in deep learning are fragile.
\newblock \emph{arXiv preprint arXiv:2006.14651}, 2020{\natexlab{a}}.

\bibitem[Basu et~al.(2020{\natexlab{b}})Basu, You, and Feizi]{basu2020second}
S.~Basu, X.~You, and S.~Feizi.
\newblock On second-order group influence functions for black-box predictions.
\newblock In \emph{International Conference on Machine Learning}, pages 715--724. PMLR, 2020{\natexlab{b}}.

\bibitem[Bellamy et~al.(2019)Bellamy, Dey, Hind, Hoffman, Houde, Kannan, Lohia, Martino, Mehta, Mojsilovi{\'c}, et~al.]{bellamy2019ai}
R.~K. Bellamy, K.~Dey, M.~Hind, S.~C. Hoffman, S.~Houde, K.~Kannan, P.~Lohia, J.~Martino, S.~Mehta, A.~Mojsilovi{\'c}, et~al.
\newblock Ai fairness 360: An extensible toolkit for detecting and mitigating algorithmic bias.
\newblock \emph{IBM Journal of Research and Development}, 63\penalty0 (4/5):\penalty0 4--1, 2019.

\bibitem[Black et~al.(2020)Black, Yeom, and Fredrikson]{black2020fliptest}
E.~Black, S.~Yeom, and M.~Fredrikson.
\newblock Fliptest: fairness testing via optimal transport.
\newblock In \emph{Proc. of FAccT}, 2020.

\bibitem[Bogen and Corbett-Davies(2021)]{meta}
M.~Bogen and S.~Corbett-Davies.
\newblock \text{What AI fairness in practice looks like at Facebook }.
\newblock https://ai.facebook.com/blog/what-ai-fairness-in-practice-looks-like-at-facebook, 2021.

\bibitem[Calders et~al.(2009)Calders, Kamiran, and Pechenizkiy]{calders2009building}
T.~Calders, F.~Kamiran, and M.~Pechenizkiy.
\newblock Building classifiers with independency constraints.
\newblock In \emph{2009 IEEE International Conference on Data Mining Workshops}, pages 13--18. IEEE, 2009.

\bibitem[Calmon et~al.(2017)Calmon, Wei, Vinzamuri, Natesan~Ramamurthy, and Varshney]{calmon2017optimized}
F.~Calmon, D.~Wei, B.~Vinzamuri, K.~Natesan~Ramamurthy, and K.~R. Varshney.
\newblock Optimized pre-processing for discrimination prevention.
\newblock In \emph{Proc. of NeurIPS}, 2017.

\bibitem[Celis et~al.(2020)Celis, Keswani, and Vishnoi]{celis2020data}
L.~E. Celis, V.~Keswani, and N.~Vishnoi.
\newblock Data preprocessing to mitigate bias: A maximum entropy based approach.
\newblock In \emph{Proc. of ICML}, 2020.

\bibitem[Celis et~al.(2021)Celis, Mehrotra, and Vishnoi]{celis2021fair}
L.~E. Celis, A.~Mehrotra, and N.~Vishnoi.
\newblock Fair classification with adversarial perturbations.
\newblock In \emph{Proc. of NeurIPS}, 2021.

\bibitem[Chai and Wang(2022)]{chai2022fairness}
J.~Chai and X.~Wang.
\newblock Fairness with adaptive weights.
\newblock In \emph{Proc. of ICML}, 2022.

\bibitem[Chen et~al.(2022)Chen, Raab, Wang, and Liu]{chen2022fairness}
Y.~Chen, R.~Raab, J.~Wang, and Y.~Liu.
\newblock Fairness transferability subject to bounded distribution shift.
\newblock \emph{arXiv preprint arXiv:2206.00129}, 2022.

\bibitem[Chen et~al.(2023)Chen, Li, Liu, and Hong]{chen2022characterizing}
Z.~Chen, P.~Li, H.~Liu, and P.~Hong.
\newblock Characterizing the influence of graph elements.
\newblock In \emph{Proc. of ICLR}, 2023.

\bibitem[Chhabra et~al.(2023)Chhabra, Li, Mohapatra, and Liu]{chhabra2022robust}
A.~Chhabra, P.~Li, P.~Mohapatra, and H.~Liu.
\newblock Robust fair clustering: A novel fairness attack and defense framework.
\newblock In \emph{Proc. of ICLR}, 2023.

\bibitem[Chouldechova(2017)]{chouldechova2017fair}
A.~Chouldechova.
\newblock Fair prediction with disparate impact: A study of bias in recidivism prediction instruments.
\newblock \emph{Big data}, 5\penalty0 (2):\penalty0 153--163, 2017.

\bibitem[Chuang and Mroueh(2021)]{chuang2021fair}
C.-Y. Chuang and Y.~Mroueh.
\newblock Fair mixup: Fairness via interpolation.
\newblock In \emph{Proc. of ICLR}, 2021.

\bibitem[Cook and Weisberg(1982)]{cook1982residuals}
R.~D. Cook and S.~Weisberg.
\newblock \emph{Residuals and influence in regression}.
\newblock New York: Chapman and Hall, 1982.

\bibitem[Ding et~al.(2021)Ding, Hardt, Miller, and Schmidt]{ding2021retiring}
F.~Ding, M.~Hardt, J.~Miller, and L.~Schmidt.
\newblock Retiring adult: New datasets for fair machine learning.
\newblock \emph{Advances in neural information processing systems}, 34:\penalty0 6478--6490, 2021.

\bibitem[du~Pin~Calmon et~al.(2018)du~Pin~Calmon, Wei, Vinzamuri, Ramamurthy, and Varshney]{du2018data}
F.~du~Pin~Calmon, D.~Wei, B.~Vinzamuri, K.~N. Ramamurthy, and K.~R. Varshney.
\newblock Data pre-processing for discrimination prevention: Information-theoretic optimization and analysis.
\newblock \emph{IEEE Journal of Selected Topics in Signal Processing}, 12\penalty0 (5):\penalty0 1106--1119, 2018.

\bibitem[Feldman(2020)]{feldman2020does}
V.~Feldman.
\newblock Does learning require memorization? a short tale about a long tail.
\newblock In \emph{Proc. of STOC}, 2020.

\bibitem[Fogliato et~al.(2020)Fogliato, Chouldechova, and G’Sell]{fogliato2020fairness}
R.~Fogliato, A.~Chouldechova, and M.~G’Sell.
\newblock Fairness evaluation in presence of biased noisy labels.
\newblock In \emph{Proc. of AISTATS}, 2020.

\bibitem[Giguere et~al.(2022)Giguere, Metevier, Brun, Castro~da Silva, Thomas, and Niekum]{giguere2022fairness}
S.~Giguere, B.~Metevier, Y.~Brun, B.~Castro~da Silva, P.~Thomas, and S.~Niekum.
\newblock Fairness guarantees under demographic shift.
\newblock In \emph{Proc. of ICLR}, 2022.

\bibitem[Gu et~al.(2019)Gu, Liu, Dolan-Gavitt, and Garg]{gu2019badnets}
T.~Gu, K.~Liu, B.~Dolan-Gavitt, and S.~Garg.
\newblock Badnets: Evaluating backdooring attacks on deep neural networks.
\newblock \emph{IEEE Access}, 7:\penalty0 47230--47244, 2019.

\bibitem[Hardt et~al.(2016)Hardt, Price, and Srebro]{hardt2016equality}
M.~Hardt, E.~Price, and N.~Srebro.
\newblock Equality of opportunity in supervised learning.
\newblock \emph{Advances in neural information processing systems}, 29:\penalty0 3315--3323, 2016.

\bibitem[He et~al.(2016)He, Zhang, Ren, and Sun]{he2016deep}
K.~He, X.~Zhang, S.~Ren, and J.~Sun.
\newblock Deep residual learning for image recognition.
\newblock In \emph{Proc. of CVPR}, 2016.

\bibitem[Hoyer et~al.(2008)Hoyer, Janzing, Mooij, Peters, and Sch{\"o}lkopf]{hoyer2008nonlinear}
P.~Hoyer, D.~Janzing, J.~M. Mooij, J.~Peters, and B.~Sch{\"o}lkopf.
\newblock Nonlinear causal discovery with additive noise models.
\newblock \emph{Advances in neural information processing systems}, 21, 2008.

\bibitem[Jang et~al.(2021)Jang, Zheng, and Wang]{jang2021constructing}
T.~Jang, F.~Zheng, and X.~Wang.
\newblock Constructing a fair classifier with generated fair data.
\newblock In \emph{Proc. of AAAI}, 2021.

\bibitem[Kamiran and Calders(2012)]{kamiran2012data}
F.~Kamiran and T.~Calders.
\newblock Data preprocessing techniques for classification without discrimination.
\newblock \emph{Knowledge and information systems}, 33\penalty0 (1):\penalty0 1--33, 2012.

\bibitem[Koh and Liang(2017)]{koh2017understanding}
P.~W. Koh and P.~Liang.
\newblock Understanding black-box predictions via influence functions.
\newblock In \emph{Proc. of ICML}, 2017.

\bibitem[Koh et~al.(2019)Koh, Ang, Teo, and Liang]{koh2019accuracy}
P.~W.~W. Koh, K.-S. Ang, H.~Teo, and P.~S. Liang.
\newblock On the accuracy of influence functions for measuring group effects.
\newblock In \emph{Proc. of NeurIPS}, 2019.

\bibitem[Kohavi et~al.(1996)]{kohavi1996scaling}
R.~Kohavi et~al.
\newblock Scaling up the accuracy of naive-bayes classifiers: A decision-tree hybrid.
\newblock In \emph{Proc. of KDD}, volume~96, pages 202--207, 1996.

\bibitem[Krasanakis et~al.(2018)Krasanakis, Spyromitros-Xioufis, Papadopoulos, and Kompatsiaris]{krasanakis2018adaptive}
E.~Krasanakis, E.~Spyromitros-Xioufis, S.~Papadopoulos, and Y.~Kompatsiaris.
\newblock Adaptive sensitive reweighting to mitigate bias in fairness-aware classification.
\newblock In \emph{Proceedings of the 2018 world wide web conference}, pages 853--862, 2018.

\bibitem[Lahoti et~al.(2020)Lahoti, Beutel, Chen, Lee, Prost, Thain, Wang, and Chi]{lahoti2020fairness}
P.~Lahoti, A.~Beutel, J.~Chen, K.~Lee, F.~Prost, N.~Thain, X.~Wang, and E.~Chi.
\newblock Fairness without demographics through adversarially reweighted learning.
\newblock \emph{Advances in neural information processing systems}, 33:\penalty0 728--740, 2020.

\bibitem[Li and Liu(2022)]{li2022achieving}
P.~Li and H.~Liu.
\newblock Achieving fairness at no utility cost via data reweighing with influence.
\newblock In \emph{Proc. of ICML}, 2022.

\bibitem[Li et~al.(2020)Li, Zhao, and Liu]{li2020deep}
P.~Li, H.~Zhao, and H.~Liu.
\newblock Deep fair clustering for visual learning.
\newblock In \emph{Proc. of CVPR}, 2020.

\bibitem[Li et~al.(2021{\natexlab{a}})Li, Wang, Zhao, Hong, and Liu]{li2021dyadic}
P.~Li, Y.~Wang, H.~Zhao, P.~Hong, and H.~Liu.
\newblock On dyadic fairness: Exploring and mitigating bias in graph connections.
\newblock In \emph{Proc. of ICLR}, 2021{\natexlab{a}}.

\bibitem[Li et~al.(2022)Li, Xia, and Liu]{li2022learning}
P.~Li, E.~Xia, and H.~Liu.
\newblock Learning antidote data to individual unfairness.
\newblock \emph{arXiv preprint arXiv:2211.15897}, 2022.

\bibitem[Li et~al.(2021{\natexlab{b}})Li, Lyu, Koren, Lyu, Li, and Ma]{li2021anti}
Y.~Li, X.~Lyu, N.~Koren, L.~Lyu, B.~Li, and X.~Ma.
\newblock Anti-backdoor learning: Training clean models on poisoned data.
\newblock In \emph{Proc. of NeurIPS}, 2021{\natexlab{b}}.

\bibitem[Linardatos et~al.(2020)Linardatos, Papastefanopoulos, and Kotsiantis]{linardatos2020explainable}
P.~Linardatos, V.~Papastefanopoulos, and S.~Kotsiantis.
\newblock Explainable ai: A review of machine learning interpretability methods.
\newblock \emph{Entropy}, 23\penalty0 (1):\penalty0 18, 2020.

\bibitem[Liu(2021)]{liu2021understanding}
Y.~Liu.
\newblock Understanding instance-level label noise: Disparate impacts and treatments.
\newblock In \emph{Proc. of ICML}, 2021.

\bibitem[Liu and Wang(2021)]{liu2021can}
Y.~Liu and J.~Wang.
\newblock Can less be more? when increasing-to-balancing label noise rates considered beneficial.
\newblock In \emph{Proc. of NeurIPS}, 2021.

\bibitem[Liu et~al.(2015)Liu, Luo, Wang, and Tang]{liu2015faceattributes}
Z.~Liu, P.~Luo, X.~Wang, and X.~Tang.
\newblock Deep learning face attributes in the wild.
\newblock In \emph{Proc. of ICCV}, December 2015.

\bibitem[Madras et~al.(2018)Madras, Creager, Pitassi, and Zemel]{madras2018learning}
D.~Madras, E.~Creager, T.~Pitassi, and R.~Zemel.
\newblock Learning adversarially fair and transferable representations.
\newblock In \emph{Proc. of ICML}, 2018.

\bibitem[Pearl(2010)]{pearl2010causal}
J.~Pearl.
\newblock Causal inference.
\newblock \emph{Causality: objectives and assessment}, pages 39--58, 2010.

\bibitem[Pruthi et~al.(2020)Pruthi, Liu, Kale, and Sundararajan]{pruthi2020estimating}
G.~Pruthi, F.~Liu, S.~Kale, and M.~Sundararajan.
\newblock Estimating training data influence by tracing gradient descent.
\newblock \emph{Advances in Neural Information Processing Systems}, 33:\penalty0 19920--19930, 2020.

\bibitem[Rezaei et~al.(2021)Rezaei, Liu, Memarrast, and Ziebart]{rezaei2021robust}
A.~Rezaei, A.~Liu, O.~Memarrast, and B.~D. Ziebart.
\newblock Robust fairness under covariate shift.
\newblock In \emph{Proc. of AAAI}, 2021.

\bibitem[Roese(1997)]{roese1997counterfactual}
N.~J. Roese.
\newblock Counterfactual thinking.
\newblock \emph{Psychological bulletin}, 121\penalty0 (1):\penalty0 133, 1997.

\bibitem[Sattigeri et~al.(2019)Sattigeri, Hoffman, Chenthamarakshan, and Varshney]{sattigeri2019fairness}
P.~Sattigeri, S.~C. Hoffman, V.~Chenthamarakshan, and K.~R. Varshney.
\newblock Fairness gan: Generating datasets with fairness properties using a generative adversarial network.
\newblock \emph{IBM Journal of Research and Development}, 63\penalty0 (4/5):\penalty0 3--1, 2019.

\bibitem[Sattigeri et~al.(2022)Sattigeri, Ghosh, Padhi, Dognin, and Varshney]{sattigeri2022fair}
P.~Sattigeri, S.~Ghosh, I.~Padhi, P.~Dognin, and K.~R. Varshney.
\newblock Fair infinitesimal jackknife: Mitigating the influence of biased training data points without refitting.
\newblock In \emph{Proc. of NeurIPS}, 2022.

\bibitem[Sharma et~al.(2020)Sharma, Zhang, R{\'\i}os~Aliaga, Bouneffouf, Muthusamy, and Varshney]{sharma2020data}
S.~Sharma, Y.~Zhang, J.~M. R{\'\i}os~Aliaga, D.~Bouneffouf, V.~Muthusamy, and K.~R. Varshney.
\newblock Data augmentation for discrimination prevention and bias disambiguation.
\newblock In \emph{Proc. of AIES}, 2020.

\bibitem[Shimizu et~al.(2006)Shimizu, Hoyer, Hyv{\"a}rinen, Kerminen, and Jordan]{shimizu2006linear}
S.~Shimizu, P.~O. Hoyer, A.~Hyv{\"a}rinen, A.~Kerminen, and M.~Jordan.
\newblock A linear non-gaussian acyclic model for causal discovery.
\newblock \emph{Journal of Machine Learning Research}, 7\penalty0 (10), 2006.

\bibitem[Song et~al.(2021)Song, Li, and Liu]{song2021deep}
H.~Song, P.~Li, and H.~Liu.
\newblock Deep clustering based fair outlier detection.
\newblock In \emph{Proc. of KDD}, 2021.

\bibitem[Song et~al.(2019)Song, Kalluri, Grover, Zhao, and Ermon]{song2019learning}
J.~Song, P.~Kalluri, A.~Grover, S.~Zhao, and S.~Ermon.
\newblock Learning controllable fair representations.
\newblock In \emph{Proc. of AISTATS}, 2019.

\bibitem[Ustun et~al.(2019)Ustun, Spangher, and Liu]{ustun2019actionable}
B.~Ustun, A.~Spangher, and Y.~Liu.
\newblock Actionable recourse in linear classification.
\newblock In \emph{Proceedings of the conference on fairness, accountability, and transparency}, pages 10--19, 2019.

\bibitem[van Breugel et~al.(2021)van Breugel, Kyono, Berrevoets, and van~der Schaar]{van2021decaf}
B.~van Breugel, T.~Kyono, J.~Berrevoets, and M.~van~der Schaar.
\newblock Decaf: Generating fair synthetic data using causally-aware generative networks.
\newblock In \emph{Proc. of NeurIPS}, 2021.

\bibitem[Verma et~al.(2020)Verma, Boonsanong, Hoang, Hines, Dickerson, and Shah]{verma2020counterfactual}
S.~Verma, V.~Boonsanong, M.~Hoang, K.~E. Hines, J.~P. Dickerson, and C.~Shah.
\newblock Counterfactual explanations and algorithmic recourses for machine learning: A review.
\newblock \emph{arXiv preprint arXiv:2010.10596}, 2020.

\bibitem[Villani et~al.(2009)]{villani2009optimal}
C.~Villani et~al.
\newblock \emph{Optimal transport: old and new}, volume 338.
\newblock Springer, 2009.

\bibitem[Wang et~al.(2021)Wang, Liu, and Levy]{wang2021fair}
J.~Wang, Y.~Liu, and C.~Levy.
\newblock Fair classification with group-dependent label noise.
\newblock In \emph{Proc. of FAccT}, 2021.

\bibitem[Wang et~al.(2022{\natexlab{a}})Wang, Wang, and Liu]{wang2022understanding}
J.~Wang, X.~E. Wang, and Y.~Liu.
\newblock Understanding instance-level impact of fairness constraints.
\newblock In \emph{Proc. of ICML}, 2022{\natexlab{a}}.

\bibitem[Wang et~al.(2022{\natexlab{b}})Wang, Dong, Xue, Zhang, Chiu, Wei, and Ren]{wang2022fairness}
Z.~Wang, X.~Dong, H.~Xue, Z.~Zhang, W.~Chiu, T.~Wei, and K.~Ren.
\newblock Fairness-aware adversarial perturbation towards bias mitigation for deployed deep models.
\newblock In \emph{Proc. of CVPR}, 2022{\natexlab{b}}.

\bibitem[Woodworth et~al.(2017)Woodworth, Gunasekar, Ohannessian, and Srebro]{woodworth2017learning}
B.~Woodworth, S.~Gunasekar, M.~I. Ohannessian, and N.~Srebro.
\newblock Learning non-discriminatory predictors.
\newblock In \emph{Conference on Learning Theory}, pages 1920--1953. PMLR, 2017.

\bibitem[Wu et~al.(2022{\natexlab{a}})Wu, Chen, Zhang, Zhu, Wei, Yuan, and Shen]{wu2022backdoorbench}
B.~Wu, H.~Chen, M.~Zhang, Z.~Zhu, S.~Wei, D.~Yuan, and C.~Shen.
\newblock Backdoorbench: A comprehensive benchmark of backdoor learning.
\newblock In \emph{Proc. of NeurIPS}, 2022{\natexlab{a}}.

\bibitem[Wu et~al.(2022{\natexlab{b}})Wu, Gong, Han, Liu, and Liu]{wu2022fair}
S.~Wu, M.~Gong, B.~Han, Y.~Liu, and T.~Liu.
\newblock Fair classification with instance-dependent label noise.
\newblock In \emph{Conference on Causal Learning and Reasoning}, pages 927--943. PMLR, 2022{\natexlab{b}}.

\bibitem[Xing et~al.(2021)Xing, Liu, Chen, and Li]{xing2021fairness}
X.~Xing, H.~Liu, C.~Chen, and J.~Li.
\newblock Fairness-aware unsupervised feature selection.
\newblock In \emph{Proc. of CIKM}, 2021.

\bibitem[Xu et~al.(2018)Xu, Yuan, Zhang, and Wu]{xu2018fairgan}
D.~Xu, S.~Yuan, L.~Zhang, and X.~Wu.
\newblock Fairgan: Fairness-aware generative adversarial networks.
\newblock In \emph{Proc. of IEEE Big Data}, 2018.

\bibitem[Zemel et~al.(2013)Zemel, Wu, Swersky, Pitassi, and Dwork]{zemel2013learning}
R.~Zemel, Y.~Wu, K.~Swersky, T.~Pitassi, and C.~Dwork.
\newblock Learning fair representations.
\newblock In \emph{Proc. of ICML}, 2013.

\bibitem[Zhang et~al.(2022)Zhang, Zhang, Zhang, Fan, Li, Liu, and Chang]{zhang2022fairness}
G.~Zhang, Y.~Zhang, Y.~Zhang, W.~Fan, Q.~Li, S.~Liu, and S.~Chang.
\newblock Fairness reprogramming.
\newblock In \emph{Proc. of NeurIPS}, 2022.

\bibitem[Zhang and Hyv{\"a}rinen(2009)]{zhang2009identifiability}
K.~Zhang and A.~Hyv{\"a}rinen.
\newblock On the identifiability of the post-nonlinear causal model.
\newblock In \emph{25th Conference on Uncertainty in Artificial Intelligence (UAI 2009)}, pages 647--655. AUAI Press, 2009.

\bibitem[Zhang and Hyvarinen(2012)]{zhang2012identifiability}
K.~Zhang and A.~Hyvarinen.
\newblock On the identifiability of the post-nonlinear causal model.
\newblock \emph{arXiv preprint arXiv:1205.2599}, 2012.

\bibitem[Zhu et~al.(2022)Zhu, Luo, and Liu]{zhu2021rich}
Z.~Zhu, T.~Luo, and Y.~Liu.
\newblock The rich get richer: Disparate impact of semi-supervised learning.
\newblock In \emph{Proc. of ICLR}, 2022.

\end{thebibliography}
\bibliographystyle{abbrvnat}

\newpage
\appendix
\onecolumn

\newpage

\section{Proposition \ref{eq:fairinfl}: Derivation of Fairness Function on Group Fairness}
\label{app:infl}

Assume the risk of $\theta$ is $R(\theta):= \frac{1}{n}\sum_{i=1}^n \ell(z^{tr}_i; \theta)$, and the model trained on the entire training set is $\hat{\theta}:=\text{argmin}_{\theta} R(\theta)$. The resulting model weights if we assign weight $w_i \in [0, 1]$ to each sample $i \in \mathcal K$ and then upweight them by some small $\epsilon$ is the following:
\begin{equation}
\hat{\theta}_{\mathcal K} := \text{argmin}_{\theta} \{R(\theta) + \epsilon \sum_{i \in \K} w_i \cdot \ell(z^{tr}_{i}; \theta)\}, 
\end{equation}
By the first order condition of $\hat{\theta}_{\mathcal K}$ we have
\[
0 = \nabla R(\hat{\theta}_{\mathcal K}) + \epsilon \sum_{i \in \K}w_i \cdot  \nabla \ell(z^{tr}_{i}; \hat{\theta}_{\mathcal K})
\]
When $\epsilon \rightarrow 0$, with the Taylor expansion (and first-order approximation) we have:
\[
0 \approx \left( \nabla R(\hat{\theta}) + \epsilon \sum_{i \in \K} w_i \cdot \nabla \ell(z^{tr}_i; \hat{\theta}) \right) + \left( \nabla^2 R(\hat{\theta}) + \epsilon \sum_{i \in \K} w_i \cdot \nabla^2 \ell(z^{tr}_i; \hat{\theta}) \right) \cdot (\hat{\theta}_{\mathcal K} - \hat{\theta})
\]
By the first-order condition of $\hat{\theta}$ we have 
 $\nabla R(\hat{\theta}) = 0$, and re-arranging terms we have 
\[
\frac{\hat{\theta}_{\mathcal K} - \hat{\theta}}{\epsilon} =  - \left(H_{\hat{\theta}}+\epsilon \cdot \sum_{i \in \K} w_i \cdot \nabla^2 \ell(z^{tr}_{i}; \hat{\theta}) \right)^{-1} \cdot \left(\sum_{i \in \K} w_i \cdot \nabla \ell(z^{tr}_i; \hat{\theta}) \right)
\]
Taking the limit of $\epsilon \rightarrow 0$ on both sides we have
\[
\frac{\partial \hat{\theta}_{\K}}{\partial \epsilon}\Big|_{\epsilon = 0} = -H^{-1}_{\hat{\theta}} \cdot \left(\sum_{i \in \K} w_i \cdot \nabla \ell(z^{tr}_i; \hat{\theta}) \right)
\]

Finally, the fairness influence of assigning training sample $i$ in group $\K$ with weight $w_i$ is:
\begin{align}
    \infl(D_{val}, \K, \hat{\theta})&:= \ell_{\text{fair}}(\hat{\theta})-\ell_{\text{fair}}(\hat{\theta}_{\K})\\
&\approx \frac{\partial \ell_{\text{fair}}(\hat{\theta}_{\K})}{\partial \epsilon} \Big|_{\epsilon = 0}\\
&=\nabla_{\theta} \ell_{\text{fair}}(\hat{\theta})^{\intercal}  \frac{\hat{\theta}_{\K}}{\partial \epsilon} \Big|_{\epsilon = 0}\\
&= -\nabla_{\theta} \ell_{\text{fair}}(\hat{\theta})^{\intercal}   H^{-1}_{\hat{\theta}}  \left(\sum_{i \in \K} 
 w_i  \nabla \ell(z^{tr}_i; \hat{\theta}) \right)
 \label{eq:group}
\end{align}

\section{Proposition \ref{eq:fairinfl-cs}: Derivation of Fairness Function for Counterfactual Samples}
\label{app:proof_fairinfl}

The proof follows largely from the one we presented above for Proposition \ref{eq:fairinfl} with the only difference being adapting the summation term to incorporate the addition of terms for counterfactual samples:
\[
\epsilon \sum_{i \in \K} w_i \cdot \ell(z^{tr}_{i}; \theta)\} \rightarrow \epsilon \sum_{i \in \K} w_i \cdot \ell(z^{tr}_{i}; \theta)\} + \epsilon' \sum_{i \in \K} w'_i \cdot \ell(\hat{z}^{tr}_{i}; \theta)\}
\]
and the counterfactual model now is defined as:
\begin{equation}
\hat{\theta}_{\mathcal K} := \text{argmin}_{\theta} \{R(\theta) + \epsilon \sum_{i \in \K} w_i \cdot \ell(z^{tr}_{i}; \theta)+\epsilon \sum_{i \in \K} w'_i \cdot \ell(\hat{z}^{tr}_{i}; \theta)\}, 
\end{equation}
Similarly invoking the first-order condition we have
\[
0 = \nabla R(\hat{\theta}_{\mathcal K}) + \epsilon \sum_{i \in \K}w_i \cdot  \nabla \ell(z^{tr}_{i}; \hat{\theta}_{\mathcal K}) + \epsilon \sum_{i \in \K}w'_i \cdot  \nabla \ell(\hat{z}^{tr}_{i}; \hat{\theta}_{\mathcal K})
\]
which further offers us
\[
\frac{\partial \hat{\theta}_{\K}}{\partial \epsilon}\Big|_{\epsilon = 0} = -H^{-1}_{\hat{\theta}} \cdot \left(\sum_{i \in \K} w_i \cdot \nabla \ell(z^{tr}_i; \hat{\theta}) +\sum_{i \in \K} w'_i \cdot \nabla \ell(\hat{z}^{tr}_i; \hat{\theta})\right)
\]
and that 
\begin{align}
    \infl(D_{val}, \K, \hat{\theta})&:= \ell_{\text{fair}}(\hat{\theta})-\ell_{\text{fair}}(\hat{\theta}_{\K}) \approx  -\nabla_{\theta} \ell_{\text{fair}}(\hat{\theta})^{\intercal}   H^{-1}_{\hat{\theta}}  \left(\sum_{i \in \K} 
 w_i  \nabla \ell(z^{tr}_i; \hat{\theta})+\sum_{i \in \K} 
 w'_i  \nabla \ell(\hat{z}^{tr}_i; \hat{\theta}) \right)
 \label{eq:group}
\end{align}

By setting the the proper $w_i, w'_i$ (e.g., $w_i = \frac{1}{n}, w'_i = -\frac{1}{n}$) we recovered the claim made in Proposition \ref{eq:fairinfl-cs}.

\section{Approximating Fairness Metrics}
\label{app:approx}

Similarly to DP, we can approximate the violation of Equality of Opportunity (EOP) with:
\begin{align}
    \ell_{EOP}(\hat{\theta}) &:= \big| \mathbb{P}(h_\theta(X) = 1 | A = 0, Y=1) - \mathbb{P}(h_\theta(X) = 1 | A = 1, Y=1) \big| \\
    &\approx \Bigg| \frac{\sum_{i \in D_{val}: a_i = 0,y_i=1} g( z^{val}_i; \theta)}{\sum_{i \in D_{val}}{\mathbb{I}[a_i = 0, y_i=1]}} - \frac{\sum_{i \in D_{val}: a_i = 1,y_i=1} g( z^{val}_i; \theta)}{\sum_{i \in D_{val}}{\mathbb{I}[a_i = 1,y_i=1]}} \Bigg|
    \label{eq:fairEOP}
\end{align}

And for Equality of Odds (EO), we have
\begin{align}
    \ell_{EO}(\hat{\theta}) &:= \frac{1}{2} \Big( \big| \mathbb{P}(h_\theta(X) = 1 | A = 0, Y=1)-\mathbb{P}(h_\theta(X) = 1 | A = 1, Y=1) \big| +  \\
    & \hspace{1.04cm} \big| \mathbb{P}(h_\theta(X) = 1 | A = 0, Y=0)-\mathbb{P}(h_\theta(X) = 1 | A = 1, Y=0) \big| \Big) \\
    &\approx \frac{1}{2} \Big( \Bigg| \frac{\sum_{i \in D_{val}: a_i = 0,y_i=1} g( z^{val}_i; \theta)}{\sum_{i \in D_{val}}{\mathbb{I}[a_i = 0, y_i=1]}} - \frac{\sum_{i \in D_{val}: a_i = 1,y_i=1} g( z^{val}_i; \theta)}{\sum_{i \in D_{val}}{\mathbb{I}[a_i = 1,y_i=1]}} \Bigg| + \\
    & \hspace{0.93cm} \Bigg| \frac{\sum_{i \in D_{val}: a_i = 0,y_i=0} g( z^{val}_i; \theta)}{\sum_{i \in D_{val}}{\mathbb{I}[a_i = 0, y_i=0]}} - \frac{\sum_{i \in D_{val}: a_i = 1,y_i=0} g( z^{val}_i; \theta)}{\sum_{i \in D_{val}}{\mathbb{I}[a_i = 1,y_i=0]}} \Bigg| \Big)
     \label{eq:fairEO}
\end{align}

We summarize the definition and surrogate approximation of three group fairness measures as follows:

\begin{table}[h]
    \centering
    \large
    \resizebox{\linewidth}{!}{
    \begin{tabular}{lcc}
    \hline
         Fairness Measure &  Definition & Surrogate Approximation \\
         \hline\hline
         Demographic Parity (DP) & $\big| \mathbb{P}(h_\theta(X) = 1 | A = 0)-\mathbb{P}(h_\theta(X) = 1 | A = 1) \big|$ & $\Bigg| \frac{\sum_{i \in D_{val}: a_i = 0} g( z^{val}_i; \theta)}{\sum_{i \in D_{val}}{\mathbb{I}[a_i = 0]}} - \frac{\sum_{i \in D_{val}: a_i = 1} g( z^{val}_i; \theta)}{\sum_{i \in D_{val}}{\mathbb{I}[a_i = 1]}} \Bigg|$ \\ [3ex]
         Equality of Opportunity (EOP) & $\big| \mathbb{P}(h_\theta(X) = 1 | A = 0, Y=1)-\mathbb{P}(h_\theta(X) = 1 | A = 1, Y=1) \big|$ & $\Bigg| \frac{\sum_{i \in D_{val}: a_i = 0,y_i=1} g( z^{val}_i; \theta)}{\sum_{i \in D_{val}}{\mathbb{I}[a_i = 0, y_i=1]}} - \frac{\sum_{i \in D_{val}: a_i = 1,y_i=1} g( z^{val}_i; \theta)}{\sum_{i \in D_{val}}{\mathbb{I}[a_i = 1,y_i=1]}} \Bigg|$ \\ [3ex]
         Equality of Odds (EO) & 
         $\text{\huge$\frac{1}{2}$} \left( \begin{array}{cc}
               \big| \mathbb{P}(h_\theta(X) = 1 | A = 0, Y=1)-\mathbb{P}(h_\theta(X) = 1 | A = 1, Y=1) \big| + \\ \hspace{-0.3cm}
             \big| \mathbb{P}(h_\theta(X) = 1 | A = 0, Y=0)-\mathbb{P}(h_\theta(X) = 1 | A = 1, Y=0) \big|
             \end{array} \right)$  & 
            $\text{\huge$\frac{1}{2}$}  \left( \begin{array}{cc}
               \big| \frac{\sum_{i \in D_{val}: a_i = 0,y_i=1} g( z^{val}_i; \theta)}{\sum_{i \in D_{val}}{\mathbb{I}[a_i = 0, y_i=1]}} - \frac{\sum_{i \in D_{val}: a_i = 1,y_i=1} g( z^{val}_i; \theta)}{\sum_{i \in D_{val}}{\mathbb{I}[a_i = 1,y_i=1]}} \big| + \\ \hspace{-0.3cm}
               \big| \frac{\sum_{i \in D_{val}: a_i = 0,y_i=0} g( z^{val}_i; \theta)}{\sum_{i \in D_{val}}{\mathbb{I}[a_i = 0, y_i=0]}} - \frac{\sum_{i \in D_{val}: a_i = 1,y_i=0} g( z^{val}_i; \theta)}{\sum_{i \in D_{val}}{\mathbb{I}[a_i = 1,y_i=0]}} \big|
            \end{array} \right) $ \\ 
            \hline
    \end{tabular}}
    \caption{Fairness definition and surrogate approximation.}
    \label{tab:fairness_approx}
\end{table}

\section{Theoretical Analysis: Why Can CIF Improve Fairness?}
\label{sec:why}

\paragraph{Overview.} We base the analysis on the data generation model adopted in \citep{feldman2020does,liu2021understanding} to capture the impact of data patterns generated with different frequencies and the impact of label errors. This setup is a good fit for understanding how counterfactual data overriding can change the data frequency of different groups (majority group with higher frequency vs. minority group with lower frequency) and provides insights for CIF.

Overriding label $Y$ is relatively straightforward. If we are able to change a training label of a disadvantaged group from a wrong label to a correct one, we can effectively improve the performance of the model for this group. Therefore the label (re)assignment can reduce the accuracy disparities. Our analysis also hints that 
the influence function is more likely to identify samples from the disadvantaged group with a lower presence in the data and mislabeled samples. This is because, for a minority group, a single label change would incur a relatively larger change in the influence value.

Overriding sensitive attribute $A$ improves fairness by balancing the data distribution.  In the experiments (Figure~\ref{fig:resample}), we show that the influence function often identifies the data from the majority group and recommends them to be changed to the minority group, as shown in Figure \ref{fig:flip}. In the analysis, we also show that this transformation incurs positive changes in the accuracy disparities between the two groups and therefore improves fairness. 
\begin{wrapfigure}{r}{0.3\textwidth}
\vspace{-0.05in}
  \begin{center}
   \includegraphics[width=0.3\textwidth]{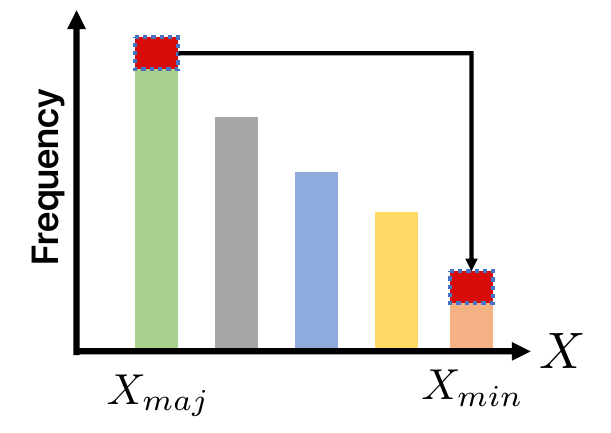}
  \end{center}
  \caption{Illustration of the effect of overriding sensitive attribute $A$ as rebalancing data distribution.}
  \label{fig:flip}
  \vspace{-0.1in}
\end{wrapfigure}

\paragraph{Setup.} We base our analysis on the data generation model adopted in \citep{feldman2020does,liu2021understanding} to capture the impact of data patterns generated with different frequencies and the impact of label errors. This setup is a good fit for understanding how counterfactual data overriding change the frequencies of data of different groups and therefore provides insights for CIF.

In this setup, each feature $X$ takes value from a \textit{discretized} set $\mathcal X$. For each $X \in \mathcal X$, sample a quantity $q_X$ independently and uniformly from a set $\lambda := \{\lambda_1,...,\lambda_N\}$.
The probability of observing an $X$ is given by $D(X)=q_X/(\sum_{X \in \mathcal X} q_X)$. Each $X$ is mapped to a true label $Y = f(X)$. But our observed training labels can be noisy, denoting as $\tilde{Y} \sim \p(\tilde{Y}|X,Y)$. $n$ pairs of $(X,\tilde{Y})$ are observed and collected for the dataset.
Denote by $S_l$ the set of all samples that appear $l$ times in the dataset, and denote by $l[X]$ the number of appearances for $X$. Each $X$ is also associated with a sensitive group attribute $A$. 
Denote by $h_{\theta}$ as the classification model defined by $\theta \in \Theta$ (parametric space) and the generalization error over a given distribution $\D$ as
$$
\err_{\D}(h_{\theta}) := \E_{\D}[\mathbf{1}(h_{\theta}(X) \neq Y)]~.
$$ The following expected generalization error is defined in \citep{feldman2020does}:
$$
\err(\theta|D):=\E_{\D \sim \p[\cdot |D]}\left[ \err_{\D}(h_{\theta})\right]~,
$$ 
where $\p[\cdot |D]$ is the distribution for the data distribution inferred from the dataset $D$. It is proved that: 
\begin{theorem}[\citep{feldman2020does}]
$
    \err(\theta|D) \geq \min_{\theta' \in \Theta}  \err(\theta'|D) + \sum_{l \in [n]} \tau_l \cdot \sum_{X \in S_l }\p[h_{\theta}(X) \neq Y].
$\label{thm:feldman}
\end{theorem}
In the above $\tau_l$ is a constant that depends on $l$. We call this the \textit{importance} of an $l$-appearance sample. It is proven in \citep{feldman2020does} that when $l$ is small, for instance $l=1$, $\tau_l$ is at the order of $O(\frac{1}{n})$, and when $l$ is large $\tau_l$ is at the order of $O(\frac{l^2}{n^2})$ \citep{liu2021understanding}.  
 
 Consider an ideal setting where we train a parametric model $\theta$ that fully memorizes the training data that $R(\theta) = 0$, and therefore $\p[h_{\theta}(X) \neq Y] =\tilde{\p}[\tilde Y \neq Y|X]$, where $\tilde{\p}[\tilde Y \neq Y|X]$ is the empirical label distribution for sample pattern $X$. Theorem \ref{thm:feldman} can easily generalize to each group $D_a$: 
 \begin{proposition}
 $
    \err(\theta|D_a) \geq \min_{\theta' \in \Theta}   \err(\theta'|D_a) + \sum_{l \in [n]} \frac{\tau_l}{\sum_{X \in D_a} \tau_{l[X]} } \cdot \sum_{X \in D_a \cap S_l }\tilde{\p}[\tilde Y \neq Y|X].
$     \label{prop:group}
 \end{proposition} 
 Denote the following \textit{excessive generalization error} for group $a$:   
 $$
 \err^{+}_a(\theta|D):= \sum_{l \in [n]} \frac{\tau_l}{\sum_{X \in D_a} \tau_{l[X]} } \cdot \sum_{X \in D_a \cap S_l }\tilde{\p}[\tilde Y \neq Y|X].$$ 
 Importantly, the above error term captures the vital quantities that are interesting to our problem: (1) the relevant frequency $\frac{\tau_l}{\sum_{X \in D_a} \tau_{l[X]} }$ captures the importance of the pattern with different frequencies and (2) $\tilde{\p}[\tilde Y \neq Y|X]$ the label noise rate of sample pattern $X$.

To set up the discussion, 
suppose we have two groups $a, a'$. $a$ is the advantaged group with a smaller $\err^{+}_a(\pi,\theta|D)$; there is an $X_a \in D_a$ with a larger $l_a$. On the other hand, there is an $X_{a'} \in D_{a'}$, an $l_{a'}$-appearance sample.  We further assume that  $l_a > l_{a'}$ ($X_{a'}$ has a lower representation).  The rest of the discussion  will focus on the following generalization error disparity as the fairness metric:
 $$
 F(\theta):= %
 |\err^{+}_a(\theta|D) - \err^{+}_{a'}(\theta|D)|~.
 $$

 The excessive generalization error for each group can be viewed as the expected influence of a model $\theta$ on the test data for that particular group. So the rest of the analysis focuses on the impact of flipping a sample's label to the group's excessive generalization error and then $F(\theta)$.
 
\para{Overriding Labels ($Y$).}
On the high level, overriding a wrong label from the disadvantaged group $a'$ to the correct one will effectively reduce $\tilde{\p}[\tilde{Y} \neq Y|X]$ for some $X \in D_{a'}$, and therefore reduces the gap from it to the advantaged groups. 
The literature on influence functions \citep{koh2017understanding} has demonstrated its power to detect mislabelled samples.
But why would the influence function identify samples from the disadvantaged group and samples with wrong labels?

Consider a specific sample $X_{a'} \in D_{a'}$, and suppose its label is wrong. Overriding the wrong label to the correct label for this rare sample leads to a reduction in noise rate $\tilde{\p}[\tilde{Y} \neq Y|X]$ for $X_{a'}$. Therefore we know that overriding this ``rare sample" reduces $ \err^{+}_{a'}(\theta|D)$, and the disparity $F(\theta)$. On the other hand, overriding the label for $X_a$ from the privileged group reduces $\err^{+}_a(\theta|D)$ but this would further increase the gap $F(\theta)$. Therefore, flipping (\ie overriding) the wrong labels from the disadvantaged group leads to a larger drop in disparity. %

\para{Overriding Sensitive Attributes ($A$).}
Suppose $X_a \in D_a$ (from the privileged group) is identified to be overridden. After the counterfactual overriding, $X_a$  is overridden to $X_{a'}$ (from the disadvantaged group), we show the gap in the excessive generalization errors between $a$ and $a'$ is reduced as follows:

\textit{(1) Increase in generalization error for the privileged group}: For group $a$'s generalization error, since we are removing one sample from it, the \textit{importance} of $X_a$ drops from $\tau_{l_a}$ to $\tau_{l_a-1}$
as
$\tau_l$ monotonically increases w.r.t $l$ (recall $\tau_l$ implies the importance of a $l$-frequency sample, the higher $l$ is the more important it generally is).  
When $X_a$ is a cleaner example that $\tilde{\p}[\tilde{Y} \neq Y|X_a]$ is sufficiently small, especially smaller than the average noise rate $\tilde{\p}[\tilde{Y} \neq Y|X \in D_a]$ of the group $a$, removing one sample of it results in an increase in the average generalization error (Proposition \ref{prop:basic}). %

\textit{(2) Decrease in generalization error for the privileged group}: For group $a'$, because of the addition, the weight of $X_{a'}$ increases by $\tau_{l_{a'}+1} -  \tau_{l_{a'}}$.
Therefore, adding a cleaner sample to group $a'$ not only reduces $X_{a'}$'s empirical label noise rate $\tilde{\p}[\tilde{Y} \neq Y|X_a]$, but also increases the relative weight of $\tau_{l_{a'}}$. Again using Proposition \ref{prop:basic}, we know that increasing the weight of a smaller quantity will then reduce the average of the group.

To summarize the above, overriding $A$ effectively (1) increases $\err^{+}_a(\theta|D)$ (\ie increasing the excessive generalization error for the privileged group) and (2) decreases $\err^{+}_{a'}(\theta|D)$ (\ie decreasing the excessive generalization error for the disadvantaged group). Therefore the counterfactual overriding $A$ reduces the gaps in the excessive generalization errors between the two groups.

\subsection{Proof of Proposition \ref{prop:group}}

Recall we assume a simplified case where we train a parametric model $\theta$ that fully memorizes the training data that $R(\theta) = 0$, and therefore $\p[h_{\theta}(X) \neq Y] =\tilde{\p}[\tilde Y \neq Y|X]$. Following the proof from \citep{feldman2020does}, it is easy to show that 
\[
\E_{\D \sim \p[\cdot |D]}\left[  \p_{\D}(h_{\theta}(X) \neq Y, X \in D_a )\right] \geq \min_{\theta' \in \Theta}   \err(\theta', X \in D_a)+ \sum_{l \in [n]}\tau_l \cdot  \sum_{X \in D_a \cap S_l} \tilde{\p}[\tilde{Y} \neq Y|X] 
\]
This is done simply by restricting generalization error to focus on data coming from a particular subset $D_a$. Note that 
\[
\E_{\D \sim \p[\cdot |D]}\left[  \p_{\D}(h_{\theta}(X) \neq Y, X \in D_a)\right] = \E_{\D \sim \p[\cdot |D]}\left[  \p_{\D}(h_{\theta}(X) \neq Y| X \in D_a) \cdot\p_{\D}( X \in D_a)  \right] 
\]
Assuming the independence of the samples drawn, we have
\begin{align*}
    \E_{\D \sim \p[\cdot |D]}\left[  \p_{\D}(h_{\theta}(X) \neq Y, X \in D_a)\right] =\   &\E_{\D \sim \p[\cdot |D]}\left[  \p_{\D}(h_{\theta}(X) \neq Y| X \in D_a)\right]  \\ &\cdot \E_{\D \sim \p[\cdot |D]}\left[\p_{\D}( X \in D_a)  \right] 
\end{align*}
From the above, we derive that 
\begin{align}
    \E_{\D \sim \p[\cdot |D]}\left[  \p_{\D}(h_{\theta}(X) \neq Y| X \in D_a)\right] = \frac{ \E_{\D \sim \p[\cdot |D]}\left[  \p_{\D}(h_{\theta}(X) \neq Y, X \in D_a)\right] }{\E_{\D \sim \p[\cdot |D]}\left[\p_{\D}( X \in D_a)  \right] }~.\label{eqn:group}
\end{align}
According to the definition of $\tau$ in \citep{feldman2020does} we have 
\begin{align*}
    \E_{\D \sim \p[\cdot |D]}\left[\p_{\D}( X \in D_a)  \right] =& \ \E_{\D \sim \p[\cdot |D]}\left[\sum_{X \in D_a} \D( X)  \right]\\
    =&\sum_{X \in D_a}  \E_{\D \sim \p[\cdot |D]}\left[ \D( X)  \right] \\
    =& \sum_{X \in D_a}  \tau_{l[X]} \tag{Definition of $\tau$}
\end{align*}
Plugging the above back into Eqn \ref{eqn:group} gives 
 $$
    \err(\theta|D_a) \geq \min_{\theta' \in \Theta}  \err(\theta'|D_a) + \sum_{l \in [n]} \frac{\tau_l}{\sum_{X \in D_a} \tau_{l[X]} } \cdot \sum_{X \in D_a \cap S_l }\tilde{\p}[\tilde Y \neq Y|X].
$$

\subsection{Basic Theorem for Proposition \ref{prop:basic}}
We next prove the following:
\begin{proposition}
    For a set of non-negative numbers $\{b_1,...,b_N\}$ with their associated non-negative weights $\{w_1,...,w_N\}$ such that $\sum_{i=1}^N w_i = 1$. Denote the average as $\bar{b}:=\sum_{i=1}^N w_i b_i$. Then
    \begin{itemize}
        \item[(1)] For $b_i < \bar{b}$, change its weight from $w_i$ to $w'_i < w_i$, and every other weight stays unchanged s.t. $w'_j = w_j$. Given the following renormalization $w'_j = \frac{w'_j}{\sum_{i} w'_i}$, we have $\bar{b}':= \sum_{i} w'_i b_i > \bar{b}$. 
        \item[(2)] For any particular $b_i < \bar{b}$, change its $b_i$  to $b'_i < b_i$ and keep other $b_j, j \neq i$ unchanged that $b'_j = b_j$. Furthermore, change its weight from $w_i$ to $w'_i > w_i$, and every other weight stays unchanged s.t. $w'_j = w_j$. Given the following renormalization $w'_j = \frac{w'_j}{\sum_{i} w'_i}$, we have $\bar{b}':= \sum_{i} w'_i b_i < \bar{b}$.
    \end{itemize}
    \label{prop:basic}
\end{proposition}
\begin{proof}
To prove (1), we have
\begin{align*}
    \bar{b}' - \bar{b} &= \sum_j (w'_j - w_j)\cdot b_j \\
    &=\sum_{j \neq i}(w'_j - w_j)\cdot b_j + \left((1-\sum_{j\neq i}w'_j) -(1-\sum_{j\neq i}w_j)  \right) b_i\\
    &=\sum_{j \neq i}(w'_j - w_j)\cdot (b_j - b_i)
\end{align*}
Furthermore, let $\Delta = w_i - w'_i$, for $j \neq i$ we have:
\[
w'_j - w_j = \frac{w_j}{1-\Delta} - w_j = w_j \cdot \frac{\Delta}{1-\Delta}
\]
Therefore we have
\begin{align*}
    \sum_{j \neq i}(w'_j - w_j)\cdot (b_j - b_i) &= \frac{\Delta}{1-\Delta} \cdot \sum_{j \neq i} w_j (b_j - b_i)\\
    &=\frac{\Delta}{1-\Delta} \cdot \left( (\bar{b}-w_i\cdot b_i)-(b_i - w_i \cdot b_i) \right)\\
    &=\frac{\Delta}{1-\Delta} \cdot (\bar{b} -b_i) > 0
\end{align*}

To prove (2), we basically follow the same proof. The only difference is that now let $\Delta = w'_i - w_i$, then for $j \neq i$:
\[
w'_j - w_j = \frac{w_j}{1+\Delta} - w_j = -w_j \cdot \frac{\Delta}{1+\Delta}
\]
Then we have
\begin{align*}
  \bar{b}' - \bar{b} &= \sum_j (w'_j - w_j)\cdot b_j + w'_i \cdot (b'_i - b_i)\\
  &= 
    \sum_{j \neq i}(w'_j - w_j)\cdot (b_j - b_i) +w'_i \cdot (b'_i - b_i)\\
    &= -\frac{\Delta}{1+\Delta} \cdot \sum_{j \neq i} w_j (b_j - b_i)+w'_i \cdot (b'_i - b_i)\\
    &=-\frac{\Delta}{1+\Delta} \cdot \left( (\bar{b}-w_i\cdot b_i)-(b_i - w_i \cdot b_i) \right)+w'_i \cdot (b'_i - b_i)\\
    &=-\frac{\Delta}{1+\Delta} \cdot (\bar{b} -b_i)+w'_i \cdot (b'_i - b_i) < 0
\end{align*}
\end{proof}

\section{Additional Experimental Results and Details}
\label{app:exp}

\subsection{Dataset Details}
\label{app:dataset}

We include the details of datasets in the following:
\squishlist
\item  \textbf{Synthetic}: We generate synthetic data with the assumed causal graphs in Figure~\ref{fig:causal}, and therefore we have the ground-truth counterfactual samples. See Appendix~\ref{app:dataset} for the dataset generation process. Model: logistic regression.
\item \textbf{COMPAS}: Recidivism prediction data (we use the preprocessed tabular data from IBM’s AIF360 toolkit~\citep{bellamy2019ai}). Feature $X$: tabular data. Label $Y$: recidivism within two years (binary). Sensitive attribute $A$ (removed from feature $X$): race (white or non-white). Model: logistic regression. When overriding $X$, we choose to flip the binary feature (age $>45$ or not) in $X$.
\item \textbf{Adult}: Income prediction data (we use the preprocessed tabular data from IBM’s AIF360 toolkit~\citep{bellamy2019ai}). Feature $X$: tabular data. Label $Y$: if income $>50K$ or not. Sensitive attribute $A$ (removed from feature $X$): sex (male or female). Model: logistic regression. When overriding $X$, we choose to flip the binary feature race (white or non-white) in $X$.
\item \textbf{CelebA}: Facial image dataset.  Feature $X$: facial images. Label $Y$: attractive or not (binary). Sensitive attribute $A$: gender (male and female). Model: ResNet18~\citep{he2016deep}. When overriding $X$, we choose to flip the binary image-level label ``Young.''

\squishend

The synthetic data is generated using a DAG with specified equations as follows:
\begin{align*}
   & X_1 \sim \text{Normal}(0,1)\\
   & A \sim \text{Bernoulli}(0.3)\\
   & X_2 \sim \text{Normal}(A, 3)\\
   & Z_1 \sim \text{Normal}(0,1)\\
   & X_3 \sim \text{Normal}(2\cdot Z_1-1, 0.1)\\
   &     X_4 \sim  \text{Bernoulli}(0.1)\\
   &     Y = sign(5 \cdot X_1 \cdot A + 0.2\cdot X^3_2+0.5 \cdot A+0.3 \cdot X_4 - X_3)
\end{align*}

We use $X_1, X_2, X_3, X_4, A$ as features, $A$ as sensitive attributes, and $Y$ as labels.

We split all tabular datasets randomly into $70\%$ training, $15\%$ validation, and $15\%$ test set. We use the original data splitting in CelebA.

\subsection{Experiment Details}
\label{app:exp_details}

We train the logistic regression on synthetic, Adult, and COMPAS using SGD with a learning rate $0.01$. For CelebA, we train ResNet18 using Adam with a learning rate $0.001$.

\para{Generating Image Counterfactual Samples.} When generating image counterfactual samples, we find directly using the generated images from W-GAN does not lead to a satisfactory mitigation performance because the distance between the counterfactual sample and the original sample is too small to impose a change that is large enough to improve fairness (tabular data has no such problem). Therefore we use a heuristic in CIF-based mitigation for image data. Using overriding $X$ as the example, when we map a sample's feature from $X|C = c$ to $X|C = c'$, we get the counterfactual feature $\hat{x}_i = G_{c_i \rightarrow \hat{c}_i}(x_i)$. We then search from the real examples $X|C = c'$ to find the nearest neighbor (in the original model's feature space) of $\hat{x}_i$, \ie
\begin{equation}
    \hat{x}_i^{'} = \argmin_{x \sim X |C = c' } ||g_{\hat{\theta}}(\hat{x}_i) - g_{\hat{\theta}}(x)||^2
\end{equation}
where $g_{\hat{\theta}}$ is the feature extractor of the original model. That is to say, we search from the pool of real samples belonging to the target group closest to the generated fake sample. Since now the counterfactual feature is another real sample in the training data, it is directly removing a sample and replace with another real sample, which induces a larger change than replacing with a fake sample that needs to be reasonably close to the original sample in the W-GAN's training constraint. The resulting counterfactual sample is $\hat{z}^{tr}_{i}(\hat{c}_i) = (\hat{x}_i^{'}, h_{\hat{\theta}}(\hat{x}_i^{'}), a_i, \hat{c}_i)$. In experiments, we cap the nearest neighbor search space to be $10\%$ of the target group size to reduce the computational cost.

\subsection{Additional Mitigation Results}
\label{app:mitigation}

\begin{figure*}[t]
  \begin{minipage}[t]{\linewidth}
    \centering
  \includegraphics[width=\linewidth]{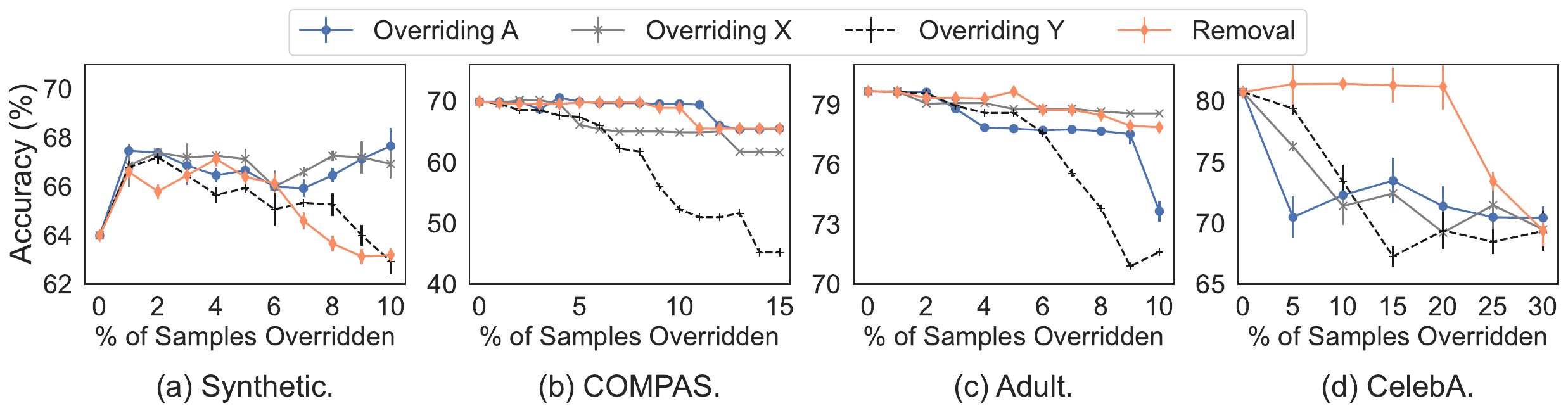}
  \caption{Model accuracy with CIF-based mitigation using fairness measure Demographic Parity (DP).}
  \label{fig:miti_acc_DP}
  \end{minipage}
  \begin{minipage}[t]{\linewidth}
    \centering
    \includegraphics[width=\linewidth]{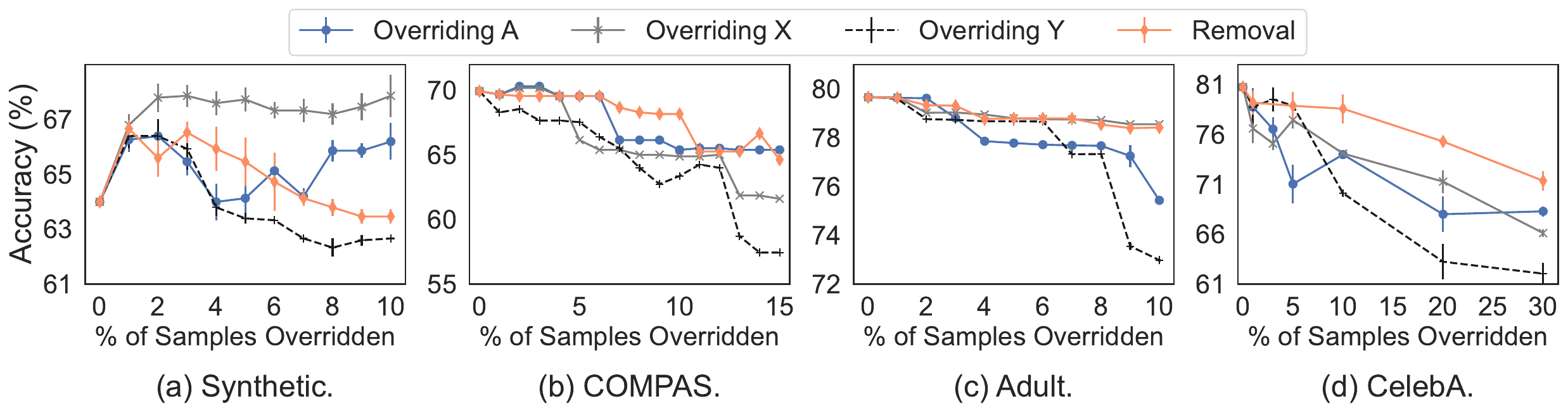}
    \caption{Model accuracy with CIF-based mitigation using fairness measure Equality of Opportunity (EOP).}
    \label{fig:miti_acc_eop}
  \end{minipage}
    \begin{minipage}[t]{\linewidth}
    \centering
    \includegraphics[width=\linewidth]{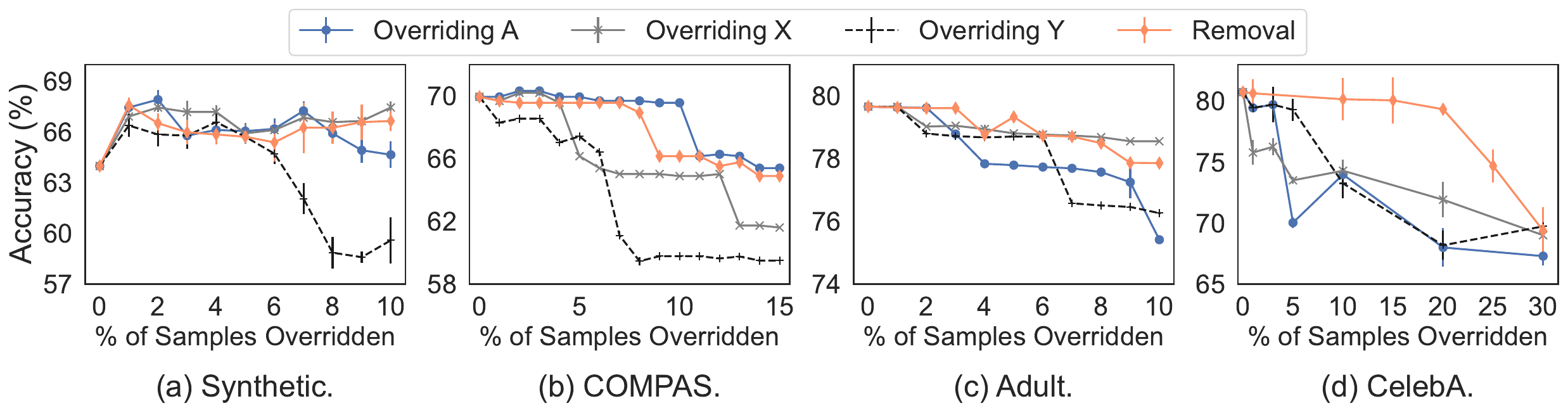}
    \caption{Model accuracy with CIF-based mitigation using fairness measure Equality of Odds (EO).}
    \label{fig:miti_acc_eo}
  \end{minipage}
\end{figure*}

Figure~\ref{fig:miti_acc_DP}-\ref{fig:miti_acc_eo} show the model accuracy after applying CIF-based mitigation.

\subsection{Impact of Label Noise}
\label{app:label_noise}
We add label noise to the synthetic data, and Table~\ref{tab:syn_label} shows the effectiveness of Y-overriding when we override 10\% top samples flagged by our method.

\begin{table}[h]
\centering
\begin{tabular}{@{}lllllll@{}}
\toprule
Noisy Rate & 0\% & 5\% & 10\% & 15\% & 20\% & 25\% \\ \midrule
Demographic Disparity & 12.4\% & 2.1\% & 6.5\% & 1.7\% & 0.05\% & 2.2\% \\ \bottomrule
\end{tabular}
\caption{Effectiveness of Y-overriding on synthetic data when label noise exists.}
\label{tab:syn_label}
\end{table}

\subsection{Accuracy of Estimated Influence value}
\label{app:infl_val}

Figure~\ref{fig:infl_est_compass} plots influence value vs. the actual difference in fairness loss (DP) on COMPAS dataset. For the first data point, we remove 100 training samples with the largest influence value from the training set, retrain the model, and compute the actual change of fairness loss (eq.\ref{eq:fairDP}) from the original model. For the next data point, we pick samples with the next 100 largest influence values and so on. We can see the relationship between our estimated influence and the actual change in fairness is largely linear, meaning our influence value can estimate the fairness change reasonably well.

\subsection{Distribution of Influence Values}
\label{app:infl_dist}
Figure~\ref{fig:infl_dist} shows the distribution of influence values computed on COMPAS corresponding to three fairness metrics.

\begin{figure}[t]
    \centering
  \includegraphics[width=\linewidth]{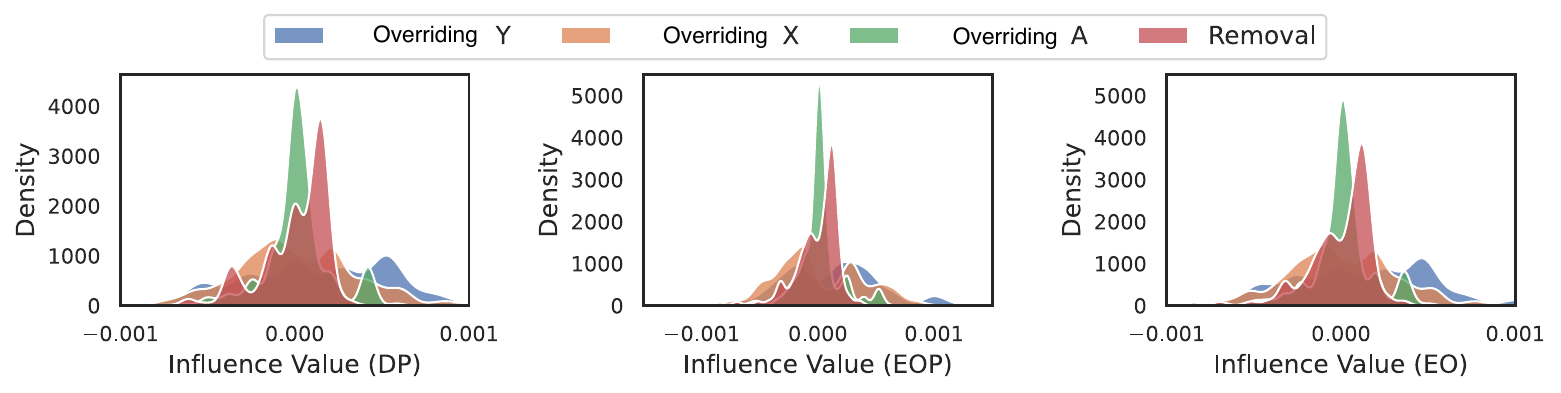}
  \caption{Distribution of influence values computed on COMPAS across three fairness metrics.}
  \label{fig:infl_dist}
\end{figure}

\subsection{Details of Experiments on Additional Applications}
\label{app:abl}

\begin{table}[!t]
\centering
\begin{tabular}{l|l|l}
\hline
      & A = 0 & A = 1 \\ \hline\hline
        Y = 0 & 0.45  & 0.35  \\ \hline
        Y = 1 & 0.15  & 0.55  \\ \hline
\end{tabular}
\caption{Group-dependent label noise rate added in the training samples in Adult data.}
\label{tab:adult_flip_y}
\end{table}

\para{Fixing Mislabelling.} The group-dependent label noise rate we add to the Adult training dataset is shown in Table~\ref{tab:adult_flip_y}. We follow a similar experimental setting in~\citep{wang2021fair}. After the label overriding, the bias increases significantly: DP increases from $16.1\%$ to $49.6\%$, EOP increases from $31.4\%$ to $77.3\%$, and EO increases from $19.1\%$ to $63.3\%$.

We flag samples by choosing samples with top influence when $Y$ is overridden and report the precision ($\frac{\#\text{flipped labels correctly detected}}{\#\text{flagged labels}}$) of our detection.

\para{Defending against Poisoning Attacks.} After the training samples are poisoned, the model unfairness increases as follows: DP increases from $16.1\%$ to $61.4\%$, EOP increases from $31.4\%$ to $69.4\%$, and EO increases from $19.1\%$ to $65.2\%$.

\para{Resampling Imbalanced Representations.} After artificially unbalancing the training samples, the fairness gap increases as follows: DP increases from $16.1\%$ to $42.6\%$, EOP increases from $31.4\%$ to $63.7\%$, and EO increases from $19.1\%$ to $47.2\%$.

\subsection{Generated Counterfactual Samples}

\begin{figure}[t]
    \centering
  \includegraphics[width=\linewidth]{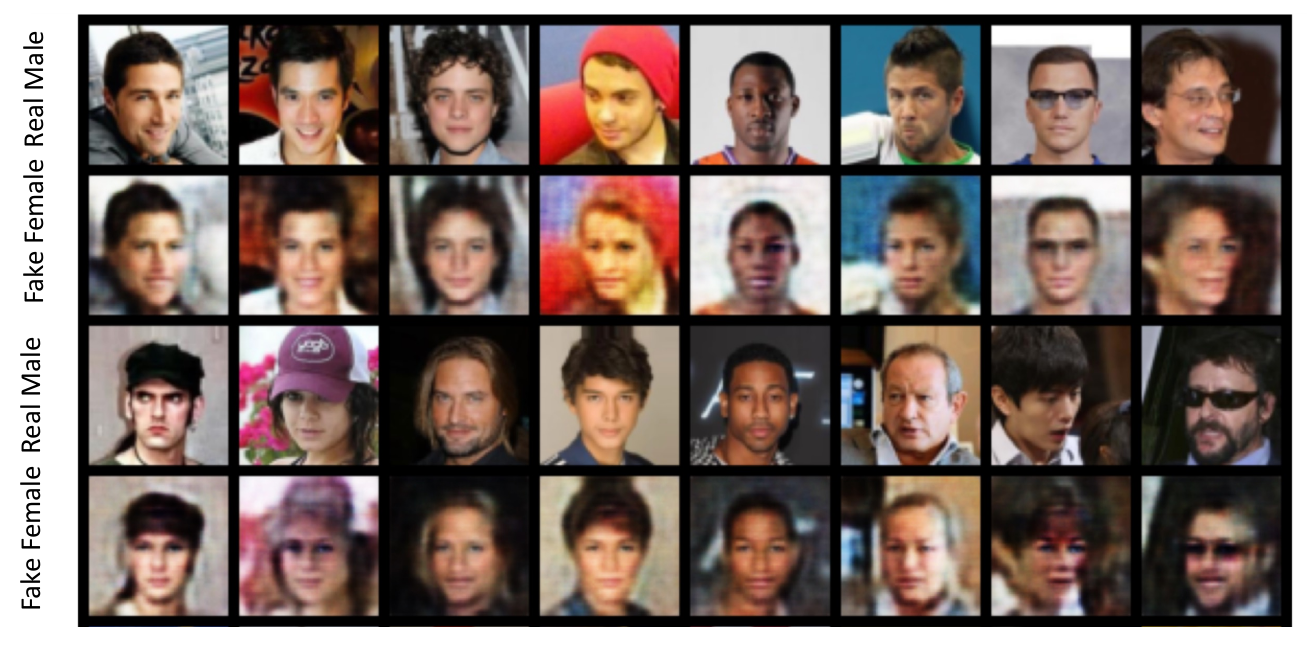}
  \caption{W-GAN generated images that map from male to female in CelebA.}
  \label{fig:gan_m2f}
\end{figure}

\begin{figure}[t]
    \centering
  \includegraphics[width=\linewidth]{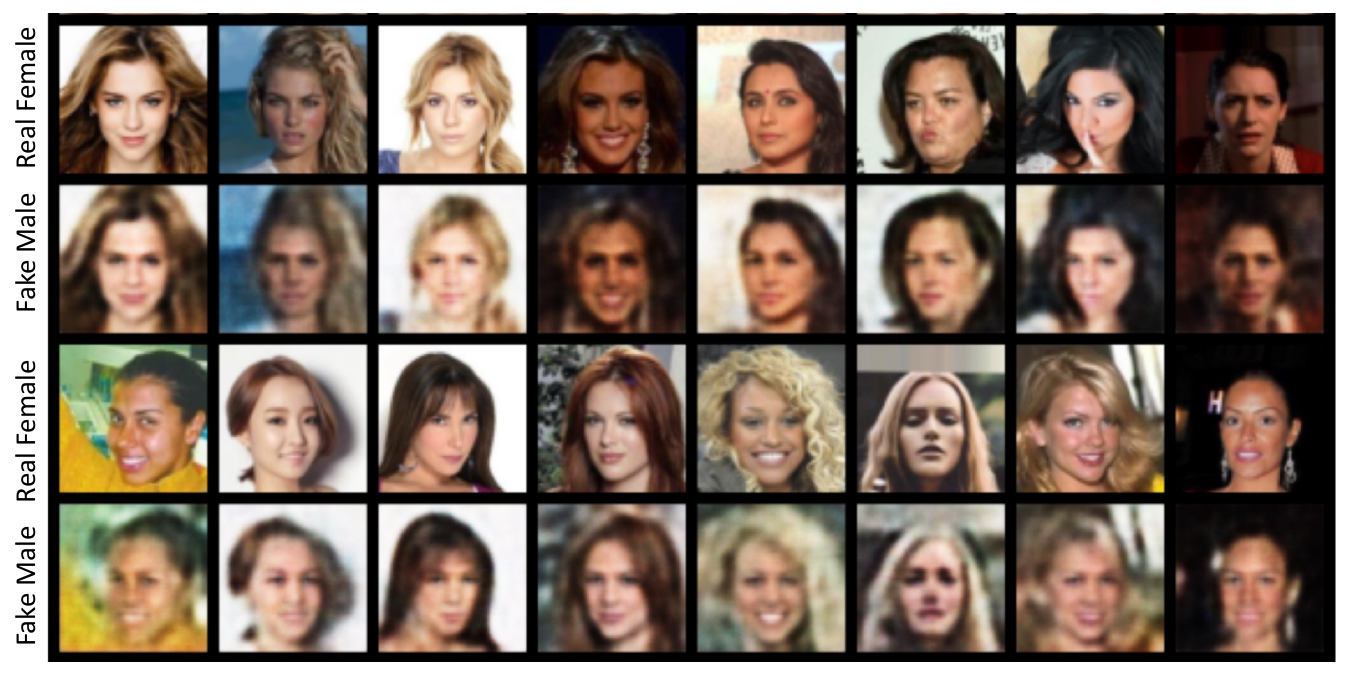}
  \caption{W-GAN generated images that map from female to male in CelebA.}
  \label{fig:gan_f2m}
\end{figure}

Figure~\ref{fig:gan_m2f} and~\ref{fig:gan_f2m} show some random examples of generated images in CelebA when overriding $A$.

\end{document}